\documentclass[10pt,a4paper,logo]{paper}

\usepackage[all]{hypcap}
\usepackage[authoryear,round]{natbib}

\usepackage{amsmath,amsfonts,bm}

\def\eqref#1{equation~\ref{#1}}

\def\1{\bm{1}}

\DeclareMathAlphabet{\mathsfit}{\encodingdefault}{\sfdefault}{m}{sl}
\SetMathAlphabet{\mathsfit}{bold}{\encodingdefault}{\sfdefault}{bx}{n}

\newcommand{\E}{\mathbb{E}}

\newcommand{\R}{\mathbb{R}}

\DeclareMathOperator*{\argmax}{arg\,max}

\usepackage{multirow}
\usepackage{mathtools}
\usepackage{makecell}
\usepackage{algorithm}
\usepackage{algpseudocode}
\usepackage{array}
\usepackage{adjustbox}
\usepackage{scalerel}
\usepackage[most,skins,theorems]{tcolorbox}
\usepackage{subcaption}
\usepackage{pdflscape}
\usepackage{float}
\usepackage{titletoc}

\titleformat{\paragraph}[runin]{\normalfont\bfseries}{}{0pt}{#1}
\titlespacing*{\paragraph}{0pt}{0.5ex plus 0.3ex minus 0.2ex}{1em}

\newtheorem{proposition}{Proposition}

\newcommand{\cA}{\mathcal{A}}
\newcommand{\cS}{\mathcal{S}}

\newcommand{\cD}{\mathcal{D}}

\newcommand{\Normal}{\mathcal{N}}
\newcommand{\dd}{\,\mathrm{d}}
\newcommand{\Aexact}{\tilde a}
\newcommand{\Ahat}{\hat a}

\newcommand{\method}{SQAM}

\usepackage[capitalise,noabbrev]{cleveref}
\crefname{appendix}{Appendix}{Appendices}
\crefname{subappendix}{Appendix}{Appendices}
\crefname{subsubappendix}{Appendix}{Appendices}
\crefname{proposition}{Proposition}{Propositions}
\crefname{algorithm}{Algorithm}{Algorithms}

\definecolor{GoogleRed}{RGB}{234, 67, 53}
\definecolor{GoogleBlue}{RGB}{66, 133, 244}
\definecolor{GoogleGreen}{RGB}{52, 168, 83}
\definecolor{ourlightblue}{RGB}{232, 240, 254}

\definecolor{title_blue}{HTML}{204899}
\definecolor{cite_blue}{HTML}{044dc1}
\hypersetup{
    colorlinks = true,
    citecolor  = {cite_blue},
    linkcolor  = {cite_blue},
    urlcolor   = {cite_blue},
}

\title{Q-Learning with Scalar Adjoint Matching}

\reportnumber{}

\author[1]{Yonghoon Dong}
\author[1,2]{Minsung Yoon}
\author[1]{Jaehyuk Kim}
\author[1]{Jungwoo Park}
\author[1]{Changyeon Kim}
\author[1,2]{Jinwoo Shin}
\affil[1]{KAIST}
\affil[2]{RLWRLD}

\correspondingauthor{Correspondence to: \href{mailto:yonghoon.dong@kaist.ac.kr,jinwoos@kaist.ac.kr}{\texttt{\{yonghoon.dong, jinwoos\}@kaist.ac.kr}}.}

\begin{abstract}
Flow policies capture rich and diverse action distributions, and fine-tuning them with off-policy RL to improve beyond the demonstrations has drawn growing interest. However, fine-tuning a flow policy against a learned value function is not trivial, because the policy generates its action over many flow steps. Adjoint matching offers a principled way to update the flow model itself by propagating value information from the final action back to each flow step, but it requires a vector--Jacobian product through the policy at every step, a cost that grows with the number of flow steps and the policy size. We observe that the batch-averaged velocity Jacobian of pretrained flow policies concentrates on its diagonal. Motivated by this finding, we derive a closed-form scalar adjoint that scales the value gradient at the final action by the flow time, eliminating the per-step vector--Jacobian products. We further find that controlling the critic's value at policy-generated actions is particularly important under the scalar adjoint. Based on these findings, we propose Q-learning with Scalar Adjoint Matching (SQAM), which combines the scalar adjoint with a value penalty at those actions. SQAM's gains concentrate on the four hardest OGBench domains, where its success rate exceeds that of the strongest baseline in each domain by 18 to 35 percentage points. To test whether SQAM extends to large pretrained policies, we also fine-tune a vision-language-action policy on a real bimanual robot. SQAM improves over supervised fine-tuning on all three tasks. 

\vspace{0.5em}
\noindent Code: \href{https://github.com/yonghdong/sqam}{\nolinkurl{github.com/yonghdong/sqam}} \quad Blog: \href{https://yonghdong.github.io/blog/sqam/}{\nolinkurl{yonghdong.github.io/blog/sqam}}
\vspace{-0.1in}
\end{abstract}

\begin{document}

\maketitle

\section{Introduction}
Flow policies represent rich and multimodal action distributions, and they have become a common
policy class for offline and offline-to-online RL~\citep{park2025flowqlearning,li2026qlearningadjointmatching,dong2026trustregionqadjoint,chi2024diffusionpolicyvisuomotorpolicy} as well as the backbone of recent large pretrained robot policies~\citep{black2025pi05,bjorck2025gr00t}.
Improving such a policy beyond the data it was trained on calls for fine-tuning it with off-policy RL
against a learned action-value function~\citep{intelligence2025pi06vlalearnsexperience,wang2026learningdeployingfleetscalereinforcement,dong2026expoftsampleefficientreinforcementlearning,xu2026rltokenbootstrappingonline,li2025grrlgoingdexterousprecise,chen2025conrft}. 

However, fine-tuning a flow policy to maximize a learned value function is not a trivial problem because of its multiple sampling procedure~\citep{park2025flowqlearning,zhang2026sacflowsampleefficientreinforcement}. Most methods therefore freeze the policy and work outside it, adding a separately learned residual to the actions it produces~\citep{dong2026expostablereinforcementlearning,xiao2025selfimprovingvisionlanguageactionmodelsdata} or learning with RL which input noise to feed the frozen policy~\citep{wagenmaker2025steeringdiffusionpolicylatent}, which leaves the pretrained policy intact during training. The improvement then rests on a separately learned component rather than on the generative model itself, so it is either bounded by what the frozen policy can express or forfeits the expressiveness of the flow policy. On the other hand, adjoint matching~\citep{domingoenrich2025adjointmatchingfinetuningflow,li2026qlearningadjointmatching,dong2026trustregionqadjoint} provides a theoretically grounded alternative for directly updating the flow model with value functions. It propagates value information from the final action back to each flow step to guide the model's updates. However, this requires a vector--Jacobian product through the policy at every step, so the computational cost grows with the number of flow steps and policy size.

In this work, we observe that the batch-averaged velocity Jacobian of a pretrained flow policy concentrates on its diagonal (\cref{fig:jacobian}). This suggests that the vector--Jacobian product adjoint matching
pays for at every flow step may not be necessary. Inspired by this finding, we approximate the
Jacobian by a scalar multiple of the identity, under which the propagation admits a closed form that
scales the value gradient at the final action by the flow time (\cref{prop:split}). We call this
the scalar adjoint. It needs a single critic gradient at the sampled action and no vector--Jacobian
product through the policy, and \cref{sec:ladder} shows that it preserves the qualitative trends
of the exact propagation.

We further find that the effectiveness of the scalar adjoint depends on the value-learning scheme. We investigate this dependence and identify the critic's value at policy-generated actions as particularly important. Our regularization ablations isolate this dependency. Methods that regulate the
critic at these actions improve performance, while methods that act only on dataset~\citep{kostrikov2021offlinereinforcementlearningimplicit} or action distance~\citep{tarasov2023revisitingminimalistapproachoffline} provide little improvement
(\cref{sec:gain}). Based on these findings, we propose \emph{Q-learning with Scalar Adjoint Matching} (\method{}), which combines the scalar adjoint with a value penalty at policy-generated actions.

Through experiments on the $50$ OGBench~\citep{ogbench_park2025} tasks, \method{} outperforms prior methods in both offline RL and offline-to-online RL, and the gains concentrate where the benchmark is hardest. On the four hardest domains it adds $18$ to $35$ percentage points of offline success rate over the best
baseline in each, and it reaches $75\%$ against $68\%$ for the strongest baseline overall. To assess whether \method{} extends to large pretrained policies, we apply it to offline-to-online RL fine-tuning of a vision-language-action model~\citep{kim2026rldx1technicalreport} on a real-world bimanual robot~\citep{rainbow_rby1} equipped with two $20$-DoF dexterous hands~\citep{wuji_hand2}. We observe that \method{} clearly improves over supervised fine-tuning on all three manipulation tasks, while residual fine-tuning~\citep{dong2026expoftsampleefficientreinforcementlearning} struggles in the same setting.

\begin{figure}[t]
\begin{center}
\includegraphics[width=0.97\linewidth]{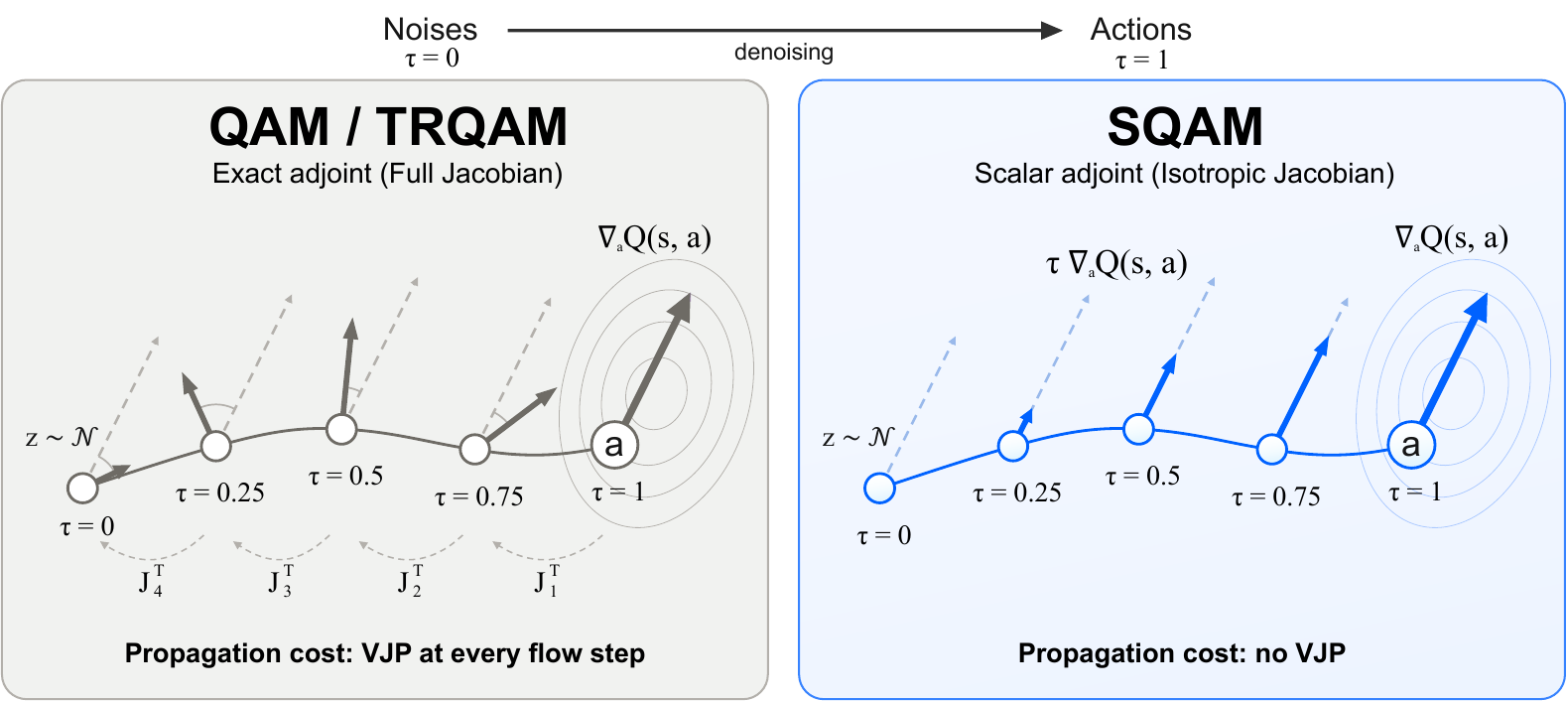}
\end{center}
\caption{Adjoint matching propagates the critic's gradient at the final action back to the flow
steps that produced it, and the propagated gradient updates the flow model. \textbf{Left:} QAM and
TRQAM integrate the lean adjoint ODE backward along the sampling chain, which takes a
vector--Jacobian product $J_i^{\top}$ through the policy at every flow step, so the cost grows with the
number of flow steps and the policy size. \textbf{Right:} \method{} instead scales the gradient at the
final action, $\nabla_aQ(s,a)$, by the flow time $\tau$, with no vector--Jacobian product at any flow
step.}
\label{fig:transport}
\vspace{-0.17in}
\end{figure}

\vspace{0.01in}
\paragraph{Contributions.}
We highlight the key contributions of our paper below:

\begin{itemize}[leftmargin=*, itemsep=2pt, parsep=0pt, topsep=0pt]
\item We introduce \method{}, an off-policy fine-tuning algorithm for flow policies that reduces the
compute cost of adjoint matching by replacing its per-step vector--Jacobian products with the
closed-form scalar adjoint.

\item We show that the scalar adjoint makes performance depend on the value learning scheme, and we
find that what matters is controlling the critic's value at the actions the current policy generates.

\item \method{}'s gains concentrate on the four hardest OGBench domains, where it exceeds the
strongest baseline by $18$ to $35$ percentage points in success rate, and it extends to fine-tuning a
pretrained vision-language-action policy on a real robot.
\end{itemize}

\section{Background}
\label{sec:setup}
\vspace{0.01in}
\paragraph{Flow matching and flow policies.}
Flow matching trains a velocity field $v_\theta(x,\tau)$ that
transports samples from a source distribution $p_0$ to a target distribution $p_1$ through
$\dd X_\tau=v_\theta(X_\tau,\tau)\dd\tau$ \citep{lipman2022flow,albergo2023stochastic,liu2022flowstraightfastlearning}. We consider the commonly used linear interpolation path $X_\tau=(1-\tau)X_0+\tau X_1$, for which the training objective is
\begin{align*}
    \mathcal{L}_{\mathrm{FM}}(\theta)\;=\;\E_{\tau,X_0,X_1}\Big[\big\|v_\theta(X_\tau,\tau)-(X_1-X_0)\big\|^2 \Big].
\end{align*}
A flow policy applies this to action generation. Given a state $s$, an action is drawn by integrating
$\dd X_\tau=v_\theta(s,X_\tau,\tau)\dd\tau$ from $X_0\sim\Normal(0,I)$ and taking the endpoint $X_1$
\citep{park2025flowqlearning}.

\vspace{0.01in}
\paragraph{Reinforcement learning.}
We consider a Markov decision process~\citep{sutton2018reinforcement} $(\cS,\cA,P,r,\gamma,\rho_0)$ with state space $\cS$, action space $\cA$, transition kernel $P:\cS\times\cA\to\Delta(\cS)$, reward $r:\cS\times\cA\to\R$, initial state distribution $\rho_0\in\Delta(\cS)$, and discount $\gamma\in[0,1)$, where $\Delta(\cdot)$ is the set of probability distributions over a set. The goal is to learn a policy $\pi:\cS\to\Delta(\cA)$ maximising the expected discounted return $\E_{\rho_0,\pi,P}\big[\sum_{t\ge0}\gamma^{t}r(s_t,a_t)\big]$. We follow the actor-critic approach, which learns an action-value function $Q^{\pi}(s,a)=\E_{\pi}\big[\sum_{k\ge0}\gamma^{k}r(s_{t+k},a_{t+k})| s_t=s,a_t=a\big]$, also called
the critic, and improves the policy by maximizing it. Our setting is off-policy fine-tuning of a pretrained flow policy. Training draws transitions $(s, a, r, s')$ from a replay buffer $\cD$, which initially holds a dataset collected by a behavior policy that may be unknown or suboptimal, and on which the flow policy $\pi_{\mathrm{base}}(\cdot|s)$ is pretrained by behavior cloning. Fine-tuning first runs offline on that data alone, and then online, where $\cD$ grows with the rollouts the fine-tuned policy collects.

\vspace{0.01in}
\paragraph{Adjoint matching.}
A flow policy generates its action over many flow steps, and the critic
only tells us how good the final action is. To improve the flow policy, each intermediate step also needs to know which direction to push its state so that the final action gets a higher value. Adjoint
matching~\citep{domingoenrich2025adjointmatchingfinetuningflow,li2026qlearningadjointmatching} computes this direction for every
step. Starting from the critic's gradient at the final action, it passes that gradient back through the
chain one step at a time by integrating the lean adjoint ODE,
\begin{align}
     \dd \Aexact_\tau= -{\Aexact_\tau}^{\!\top}\nabla_{X_\tau}\Big(2v^{\mathrm{base}}(X_\tau,\tau)-\frac{1}{\tau}X_\tau\Big) \dd \tau,
    \qquad \Aexact_1=-\nabla_{X_1}Q^\pi(s,X_1) ,
    \label{eq:adjoint}
\end{align}
and trains the fine-tuned velocity $v^{\mathrm{ft}}_\theta$ to match the gradient that \cref{eq:adjoint}
delivers to each step,
\begin{align}
    \mathcal{L}_{\mathrm{Adj-Match}}(\theta)\;=\;\E\bigg[\sum_\tau\Big\|\frac{2}{\sigma(\tau)}\big(v^{\mathrm{ft}}_\theta(X_\tau,\tau)-v^{\mathrm{base}}(X_\tau,\tau)\big)+\sigma(\tau)\Aexact_\tau\Big\|^2\bigg] ,
    \label{eq:am}
\end{align}
where $\sigma(\tau)$ is the noise schedule of the sampler. \citet{dong2026trustregionqadjoint} stabilize the
optimization by constraining the fine-tuned sampler to a path-space KL budget
$\varepsilon_{\mathrm{KL}}$ (\cref{app:baselines}).

\section{SQAM: Q-Learning with Scalar Adjoint Matching}
\label{sec:method}

\subsection{A closed-form scalar adjoint under an isotropic velocity Jacobian}
\label{sec:decomp}

Solving the lean adjoint ODE in \cref{eq:adjoint} requires a vector--Jacobian product with the
velocity Jacobian $\nabla_{X_\tau}v^{\mathrm{base}}(X_\tau,\tau)$ at every flow step, so the cost of
each update grows with both the number of flow steps and the policy size, which makes adjoint matching
difficult to scale as pretrained flow policies grow
larger~\citep{black2025pi05,bjorck2025gr00t}. We resolve this problem with the empirical finding that the batch-averaged velocity Jacobian
$J_\tau:=\nabla_{X_\tau}v^{\mathrm{base}}(X_\tau,\tau)$ of the pretrained policy concentrates on its diagonal (\cref{fig:jacobian}, with all ten domains in
\cref{app:geometry}). Following the isotropic approximation of the posterior covariance
in \citet{peng2024improving}, we approximate $J_\tau=c_\tau I$ with a scalar $c_\tau$, which
solves \cref{eq:adjoint} in closed form.

\begin{proposition}[the lean adjoint ODE under an isotropic velocity Jacobian]
\label{prop:split}
Assume $J_\tau=c_\tau I$ for a scalar $c_\tau$. Then the lean adjoint ODE in \cref{eq:adjoint} admits the closed-form solution
\begin{align*}
    \Aexact_\tau\;=\;-\,\tau\,\exp\!\Big(2\!\int_\tau^1 c_s\,\dd s\Big)\,\nabla_{X_1}Q^\pi(s,X_1),
\end{align*}
which points along the negative of the critic's action gradient at the final action at every flow time
$\tau>0$.
\end{proposition}
\begin{proof}
With $J_\tau=c_\tau I$, the Jacobian term in \cref{eq:adjoint} becomes
$\big(\frac{1}{\tau}-2c_\tau\big)I$, so the ODE reads
$\dd\Aexact_\tau=\big(\frac{1}{\tau}-2c_\tau\big)\Aexact_\tau\,\dd\tau$ and every coordinate is
scaled by the same factor. Integrating from $\tau$ to $1$ gives
$\Aexact_\tau=\tau\exp\big(2\int_\tau^1 c_s\,\dd s\big)\Aexact_1$, and the terminal condition is
$\Aexact_1=-\nabla_{X_1}Q^\pi(s,X_1)$.
\end{proof}

At every flow time, the closed-form solution is a scalar multiple of the critic's action
gradient at the final action. In practice, estimating $c_\tau$ would re-introduce the Jacobian we set out
to avoid, so we eliminate the exponential factor and keep only
\begin{align*}
    \Ahat_\tau
    =-\tau\nabla_{X_1}Q^\pi(s,X_1),
    \qquad \forall \tau\in[0,1],
\end{align*}
which we call \emph{scalar adjoint}. The exponential factor $\exp\big(2\int_\tau^1 c_s\,\dd s\big)$ only rescales the closed-form solution, and dropping it is empirically justified by
\cref{fig:kappa} (left), where the norm of $\Ahat_\tau$ follows the same trend as that of
$\Aexact_\tau$ across flow times. The scalar adjoint requires a single critic gradient at the sampled action $X_1$, and no
vector--Jacobian product through the policy at any flow step. 

\begin{figure}[t]
\begin{center}
\includegraphics[width=0.95\linewidth]{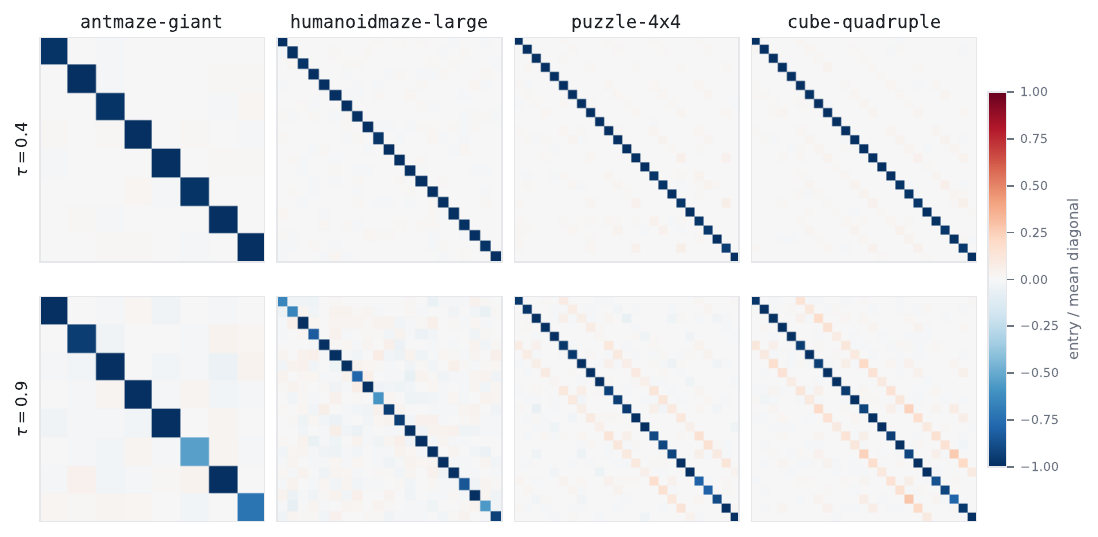}
\end{center}
\caption{Velocity Jacobian of the pretrained flow on four domains at two flow times, averaged over a
batch of states and normalized per panel by the mean diagonal. The batch-averaged Jacobian is
dominated by its diagonal.}
\label{fig:jacobian}
\vspace{-0.15in}
\end{figure}

Moreover, \method{} follows the adjoint matching form of \cref{eq:am}, replacing the exact adjoint
$\Aexact_\tau$ in its target with the scalar adjoint $\Ahat_\tau$, so the path-space trust region of TRQAM~\citep{dong2026trustregionqadjoint} can be used or left out. We use it by
default for stable training, and \cref{app:notrust} reports \method{} without it.
\cref{app:cost} compares the computational cost of the scalar and exact adjoint updates.

\subsection{Regularizing the critic at the policy's own actions}
\label{sec:critic}

Using the scalar adjoint alone does not work well, especially on the manipulation domains
(\cref{sec:gain}). The exact adjoint passes $\nabla_{X_1}Q^\pi(s,X_1)$ through the velocity Jacobian at every flow step,
so it does not use this gradient as directly as the scalar adjoint does (\cref{fig:transport}). The
scalar adjoint update therefore depends directly on the critic's gradient at the policy-generated
action, and we hypothesize that its failure arises from errors in the critic at that action. To address this, we introduce a
value penalty on the critic at the actions the current policy generates, relative to the value at the
dataset action,
\begin{align}
    \mathcal{L}_{\mathrm{critic}} \;\mathrel{+}=\;
    c\cdot\big(\,Q^\pi(s,\operatorname{sg}(a_\pi))
    -Q^\pi(s,a_{\mathrm{data}})\,\big).
    \label{eq:gapreg}
\end{align}
\cref{eq:gapreg} penalizes the critic at the policy-generated action $a_\pi$, which is
also the action whose critic gradient drives the scalar adjoint update, so the penalty controls the
critic exactly where the policy update uses it. We refer to this regularizer as the \emph{value
penalty}. Unlike CQL~\citep{kumar2020conservativeqlearningofflinereinforcement}, which penalizes the
critic in expectation over a sampling distribution drawn inside the critic loss, the value penalty
acts at the single action the actor update has already drawn, so it adds no sampling cost.

Common forms of critic regularization act elsewhere. In-sample
maximization~\citep{kostrikov2021offlinereinforcementlearningimplicit,garg2023extremeqlearningmaxentrl} fits the critic only on dataset actions, so it never evaluates the critic at $a_\pi$. Behavioral
regularization~\citep{tarasov2023revisitingminimalistapproachoffline,fujimoto2021minimalistapproachofflinereinforcement} penalizes the distance between $a_\pi$ and $a_{\mathrm{data}}$, which constrains the actions rather than the value assigned to them. Neither controls the critic at the action whose gradient the scalar adjoint uses, and the value penalty does. Experiments in \cref{sec:gain} support this.

\begin{figure}[t]
\centering
\begin{minipage}[t]{0.495\linewidth}
\begin{algorithm}[H]
\caption{TRQAM \citep{dong2026trustregionqadjoint}}
\label{alg:trqam}
\begin{algorithmic}[1]
\small
\Require $v^{\mathrm{base}}$, $Q^\pi_\phi$, KL budget $\varepsilon_{\mathrm{KL}}$, training steps $N$
\State $v^{\mathrm{ft}}_\theta\leftarrow v^{\mathrm{base}}$
\For{$n=0,\ldots,N-1$}
    \State Sample a trajectory $X_0,\dots,X_1$ by $v^{\mathrm{ft}}_\theta$
    \State $X_\tau\leftarrow$ the states of that trajectory
    \State Solve $\dd\Aexact_\tau=-{\Aexact_\tau}^{\!\top}\nabla_{X_\tau}\big(2v^{\mathrm{base}}-\frac{1}{\tau}X_\tau\big)\dd\tau$
    \State Update $\theta$ by \cref{eq:am}
    \State TD update of $Q^\pi_\phi$
    \State Dual update of $\lambda_n$ for trust region
\EndFor
\end{algorithmic}
\end{algorithm}
\end{minipage}
\hfill
\begin{minipage}[t]{0.495\linewidth}
\begin{algorithm}[H]
\caption{\textcolor{GoogleBlue}{SQAM (ours)}}
\label{alg:sqam}
\begin{algorithmic}[1]
\small
\Require $v^{\mathrm{base}}$, $Q^\pi_\phi$, KL budget $\varepsilon_{\mathrm{KL}}$, training steps $N$
\State $v^{\mathrm{ft}}_\theta\leftarrow v^{\mathrm{base}}$
\For{$n=0,\ldots,N-1$}
    \State Sample the endpoint $X_1$ by $v^{\mathrm{ft}}_\theta$
    \State \textcolor{GoogleBlue}{$X_\tau\leftarrow(1-\tau)\,\epsilon+\tau X_1,\ \epsilon\sim\mathcal N(0,I)$}
    \State \textcolor{GoogleBlue}{$\Ahat_\tau\leftarrow-\tau\,\nabla_{X_1}Q^\pi(s,X_1),\ \forall\tau\in[0,1]$}
    \State Update $\theta$ by \cref{eq:am}
    \State TD update of $Q^\pi_\phi$ \textcolor{GoogleBlue}{${}+{}$\cref{eq:gapreg}}
    \State Dual update of $\lambda_n$ for trust region
\EndFor
\end{algorithmic}
\end{algorithm}
\end{minipage}
\caption{Simplified versions of the TRQAM and \method{} algorithms. The parts \method{} changes
are marked in \textcolor{GoogleBlue}{blue}, and \cref{app:algorithm} gives the full
algorithm.}
\end{figure}

\section{Experiments}
\label{sec:exp}

We evaluate \method{} on off-policy fine-tuning of pretrained flow policies in the offline and
offline-to-online settings, using the 50 OGBench tasks for the main comparison. The experiments are organized around the following five questions.
\begin{itemize}[leftmargin=*, itemsep=2pt, parsep=0pt, topsep=0pt]
\item How does \method{} perform against prior methods in offline and offline-to-online RL
(\cref{sec:main})?
\item How does the scalar adjoint compare with the exact adjoint (\cref{sec:ladder})?
\item Why does the value penalty work better than other value learning schemes
(\cref{sec:gain})?
\item How do the value penalty coefficient $c$ and the KL budget $\varepsilon_{\mathrm{KL}}$ affect
\method{} (\cref{sec:sensitivity})?
\item Does \method{} extend to large pretrained policies (\cref{sec:scale})?
\end{itemize}

\subsection{Main Experiments}
\label{sec:main}

\paragraph{Setup.} OGBench \citep{ogbench_park2025} is a goal-conditioned RL benchmark, and we use its reward-based single-task variants over 10 domains with 5 tasks per domain, 50 tasks in all. Manipulation domains use action-chunked policies of chunk size $5$~\citep{li2025reinforcementlearningactionchunking}. All methods start from the same flow policy, pretrained with behavior cloning for 300K steps, and are fine-tuned offline for 1M steps with online runs continuing to 1.5M steps. We report the average success rate (\%) over 8 seeds, and the detailed hyperparameters are in \cref{app:hyperparams}. We compare against seven off-policy fine-tuning methods for flow policies, FQL \citep{park2025flowqlearning}, CGQL-L~\citep{dhariwal2021diffusion}, DSRL \citep{wagenmaker2025steeringdiffusionpolicylatent}, IFQL, QAM and QAM-E \citep{li2026qlearningadjointmatching}, and TRQAM \citep{dong2026trustregionqadjoint}. The last three belong to the same adjoint matching family as \method{}. See \cref{app:baselines} for the detailed description of each baseline.

\paragraph{Results.} \cref{tab:offline} reports offline success rates at 1M steps. \method{} improves most on the four hardest domains, adding $21$ points over the best baseline on \texttt{antmaze-giant}, $18$
on \texttt{humanoidmaze-large}, $22$ on \texttt{cube-triple}, and $35$ on \texttt{cube-quadruple},
where the strongest baseline reaches $19$ and \method{} reaches $54$. Over all fifty tasks it reaches
$75\%$ against $68\%$ for TRQAM, the strongest baseline. This trend holds through the online phase as well. The per-task breakdown of all fifty tasks is in \cref{tab:full-results}, and their success rate curves are in \cref{app:benchmark}.

\begin{table}[t]
    \centering
    \renewcommand{\arraystretch}{1.2}
    \setlength{\tabcolsep}{3pt}
    \setlength{\abovecaptionskip}{4pt}
    \caption{\textbf{Offline RL on 50 OGBench~\citep{ogbench_park2025} tasks
    at 1M training steps (8 seeds).}
    Mean success rate (\%) $\pm$ one standard deviation.
    Domain abbreviations are listed in \cref{app:domains}, and the per-task breakdown across all 50
    tasks is in \cref{tab:full-results}.}
    \label{tab:offline}
    \newcommand{\gval}[2]{{\color{gray!70}#1}\,{\color{gray!50}\scriptsize$\pm$#2}}
    \newcommand{\bmax}[2]{\underline{#1}\,{\color{gray!50}\scriptsize$\pm$#2}}
    \newcommand{\gavg}[1]{{\color{gray!70}#1}}
    \newcommand{\ourmax}[2]{{\color{GoogleBlue}\textbf{\underline{#1}}}\,{\color{gray!50}\scriptsize$\pm$#2}}
    \newcommand{\ourval}[2]{{\color{GoogleBlue}#1}\,{\color{gray!50}\scriptsize$\pm$#2}}
    \newcommand{\ouravg}[1]{{\color{GoogleBlue}\textbf{\underline{#1}}}}
\resizebox{0.9\linewidth}{!}{
\begin{tabular}{>{\centering\arraybackslash}m{2.2cm} l ccccccccccc}
\toprule
 &  & \texttt{al} & \texttt{ag} & \texttt{hm} & \texttt{hl} & \texttt{scene} & \texttt{p33} & \texttt{p44} & \texttt{c2} & \texttt{c3} & \texttt{c4} & \texttt{all} \\
 &  & \tiny\texttt{5 tasks} & \tiny\texttt{5 tasks} & \tiny\texttt{5 tasks} & \tiny\texttt{5 tasks} & \tiny\texttt{5 tasks} & \tiny\texttt{5 tasks} & \tiny\texttt{5 tasks} & \tiny\texttt{5 tasks} & \tiny\texttt{5 tasks} & \tiny\texttt{5 tasks} & \tiny\texttt{50 tasks} \\
\midrule
\textsc{Backprop} & {\texttt{FQL}}
& \gval{38}{9} & \gval{2}{6} & \gval{74}{5} & \gval{2}{1} & \gval{70}{5} & \gval{25}{10} & \gval{9}{7} & \gval{44}{4} & \gval{7}{5} & \gval{9}{5} & \gavg{28} \\
\midrule
\textsc{Guidance} & {\texttt{CGQL-L}}
& \gval{48}{7} & \gval{7}{5} & \gval{57}{2} & \gval{6}{3} & \gval{58}{1} & \gval{0}{0} & \gval{0}{0} & \gval{55}{2} & \gval{0}{1} & \gval{1}{1} & \gavg{23} \\
\midrule
\textsc{Post Processing} & {\texttt{DSRL}}
& \gval{53}{2} & \gval{1}{1} & \gval{53}{10} & \gval{1}{1} & \bmax{80}{0} & \bmax{100}{0} & \gval{61}{8} & \gval{72}{4} & \gval{34}{6} & \gval{9}{3} & \gavg{46} \\
& {\texttt{IFQL}}
& \gval{29}{8} & \gval{12}{3} & \gval{93}{2} & \gval{30}{7} & \gval{36}{1} & \gval{64}{4} & \gval{42}{4} & \gval{9}{2} & \gval{24}{7} & \gval{6}{3} & \gavg{35} \\
\midrule
\textsc{Adjoint Matching} & {\texttt{QAM}}
& \gval{62}{9} & \gval{29}{4} & \gval{64}{7} & \gval{4}{3} & \gval{64}{4} & \gval{15}{3} & \gval{1}{1} & \gval{71}{2} & \gval{19}{6} & \gval{18}{3} & \gavg{35} \\
& {\texttt{QAM-E}}
& \gval{86}{3} & \gval{6}{8} & \gval{60}{6} & \gval{4}{5} & \gval{63}{6} & \gval{89}{4} & \gval{54}{8} & \gval{71}{3} & \gval{11}{4} & \gval{9}{3} & \gavg{45} \\
& {\texttt{TRQAM}}
& \gval{89}{4} & \gval{41}{4} & \gval{84}{3} & \gval{36}{4} & \gval{79}{1} & \bmax{100}{0} & \bmax{99}{1} & \bmax{81}{3} & \gval{50}{5} & \gval{19}{5} & \gavg{68} \\
\midrule
\rowcolor{ourlightblue}
\textsc{\textbf{Ours}} & \texttt{SQAM}
& \ourmax{91}{3} & \ourmax{62}{4} & \ourmax{94}{2} & \ourmax{54}{5} & \ourmax{79}{0} & \ourmax{100}{0} & \gval{85}{6} & \gval{58}{3} & \ourmax{72}{4} & \ourmax{54}{5} & \ouravg{75} \\
\bottomrule
\end{tabular}
}
\end{table}

\subsection{Comparison between Adjoints}
\label{sec:ladder}
\cref{fig:kappa} explains why the scalar adjoint works even though it approximates the exact
adjoint. The approximation preserves the two properties the update relies on. In magnitude, both
adjoints contract in the same way as the flow time approaches zero, so the regression target in
\cref{eq:am} has a similar magnitude at each flow step under either adjoint. In direction, the scalar
adjoint stays positively aligned with the exact adjoint at every flow time. These two properties allow
the scalar adjoint to stand in for the exact adjoint, but they do not make the two identical. The angle
between them is larger on the two manipulation domains than on the two locomotion domains, and these
are exactly the domains where the scalar adjoint alone underperforms (\cref{sec:critic}). Since
the scalar adjoint relies directly on the critic's gradient, this correspondence led us to examine the
role of the value learning scheme in this framework, which \cref{sec:gain} presents. 

\subsection{Analyses over Value Learning Schemes}
\label{sec:gain}
To justify the choice of the value penalty, we compare it with three alternative value learning schemes. Two are widely used forms of critic regularization, and the third directly regularizes the critic's action gradient that the scalar adjoint uses. The first is an in-sample critic in the style of IQL~\citep{kostrikov2021offlinereinforcementlearningimplicit}, which fits a state value by expectile regression on dataset actions and trains the critic by TD regression toward it. The second is ReBRAC~\citep{tarasov2023revisitingminimalistapproachoffline}'s bootstrap penalty, which subtracts
$\lVert a_\pi-a_{\mathrm{data}}\rVert^2$ from the bootstrap target. The third penalizes the norm of the critic's action gradient, adding $\big\|\nabla_a Q^\pi(s,\operatorname{sg}(a_\pi))\big\|^2$ to the critic loss. See \cref{app:consforms} for the detailed description of each scheme. 

\cref{fig:conservatism} compares the schemes on \texttt{cube-triple} and \texttt{cube-quadruple}, the two manipulation domains where the scalar adjoint alone underperforms, and \cref{app:schemesfull} reports every scheme swept over its coefficient on all ten domains. The value penalty works best on both, while the in-sample critic and ReBRAC stay close to the scalar adjoint without regularization. This follows from where each scheme acts. The scalar adjoint uses the critic's gradient at the policy-generated action $a_\pi$ directly at every flow step, so the policy update depends directly on the critic at $a_\pi$. On the other hand, the in-sample critic is trained only on dataset actions, so it never evaluates the critic at $a_\pi$. ReBRAC constrains the distance between $a_\pi$ and $a_{\mathrm{data}}$, which controls the action itself rather than the value the critic assigns to it. Only the value penalty acts on the critic's value at $a_\pi$.

The gradient-norm penalty also improves over the unregularized critic, most clearly on \texttt{cube-quadruple}, which is consistent with its close relation to the value penalty. A first-order expansion of the value penalty in \cref{eq:gapreg} around $a_\pi$ gives
$$Q^\pi(s,a_\pi)-Q^\pi(s,a_{\mathrm{data}})\approx
\nabla_{a_\pi}Q^\pi(s,a_\pi)^{\!\top}(a_\pi-a_{\mathrm{data}}),$$
so the two penalties are linked through the critic's gradient at $a_\pi$.

The explanation above rests on the scalar adjoint using the critic's gradient at $a_\pi$ directly.
The exact adjoint instead passes this gradient through the velocity Jacobian at every flow step, so if
the explanation holds, the value penalty should help the exact adjoint less. To test this, we also
apply the value penalty to TRQAM, which uses the exact adjoint. The value penalty improves TRQAM as well, but far less than it improves \method{} (\cref{app:chainpen}).

\begin{figure}[t]
\centering
\includegraphics[width=0.93\linewidth]{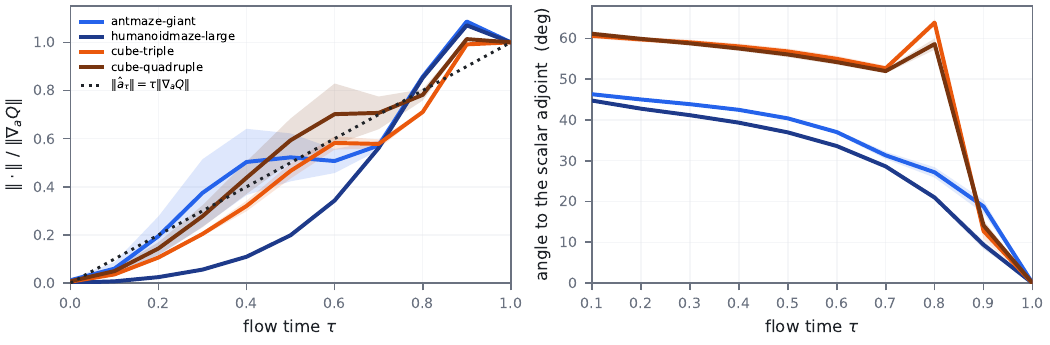}
\caption{Exact and scalar adjoints computed on the same rollout states. \textbf{Left:} the norm of
each adjoint relative to $\|\nabla_aQ\|$, where the scalar adjoint is the dotted line $\tau$.
\textbf{Right:} the angle between the two adjoints. The two adjoints contract similarly in norm and stay positively aligned. The angle is larger on the manipulation domains, where the scalar adjoint alone underperforms. Three seeds per domain at 0.3M steps, without the value penalty.}
\label{fig:kappa}\label{fig:norm}
\vspace{-0.15in}
\end{figure}

\begin{figure}[t]
\begin{center}
\includegraphics[width=0.93\linewidth]{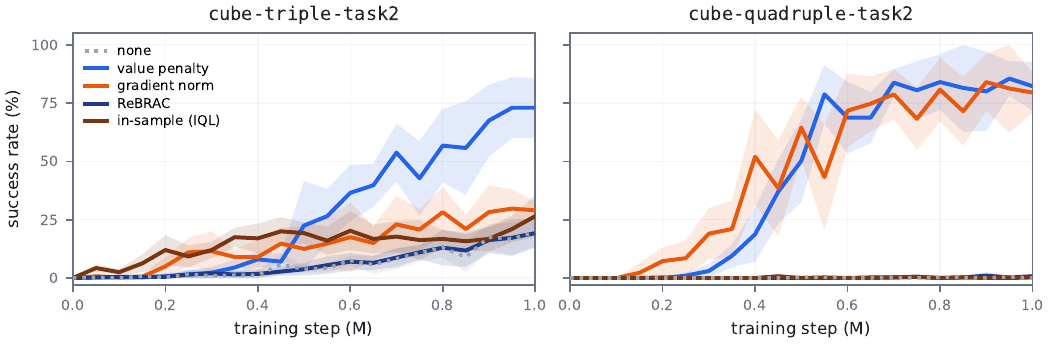}
\end{center}
\caption{Value learning schemes compared on the two domains where the scalar adjoint alone
fails. Each scheme is shown at its best coefficient on the grid we swept, and the full sweeps on all ten
domains are in \cref{app:schemesfull}. The value penalty lifts both domains and the gradient-norm penalty follows it on \texttt{cube-quadruple}, while ReBRAC and the in-sample critic stay close to the scalar adjoint without regularization. Eight seeds per curve, and bands are one standard deviation.}
\label{fig:conservatism}
\vspace{-0.15in}
\end{figure}

\subsection{Ablation Studies}
\label{sec:sensitivity}

\vspace{0.01in}
\paragraph{Effect of value penalty $c$.}

\cref{fig:coefcurves} assesses the value penalty coefficient $c$ on one task from each of four
domains. On the two cube manipulation tasks, the success rate rises sharply as $c$ grows, while on
the two locomotion tasks the gains are small. These results indicate that the value penalty matters
mainly on the domains where the scalar adjoint alone underperforms. We use $c=0.3$ on every domain
except \texttt{cube-double}, where the penalty is off because a large $c$ lowers its offline
performance (\cref{app:sensfull}).

\begin{figure}[t]
\begin{center}
\includegraphics[width=0.96\linewidth]{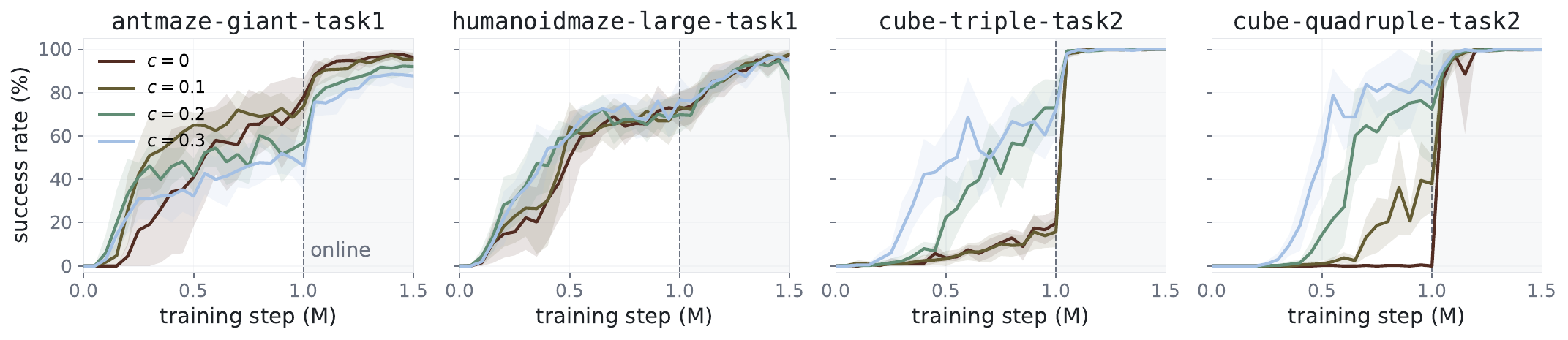}
\end{center}
\vspace{-0.1in}
\caption{Sweep over value penalty $c$ on four of the ten swept domains. Eight seeds per curve,
and bands are one standard deviation. The full analysis on all ten domains is in \cref{app:sensfull}.}
\label{fig:coefcurves}
\vspace{-0.1in}
\end{figure}

\begin{figure}[t]
\begin{center}
\includegraphics[width=0.96\linewidth]{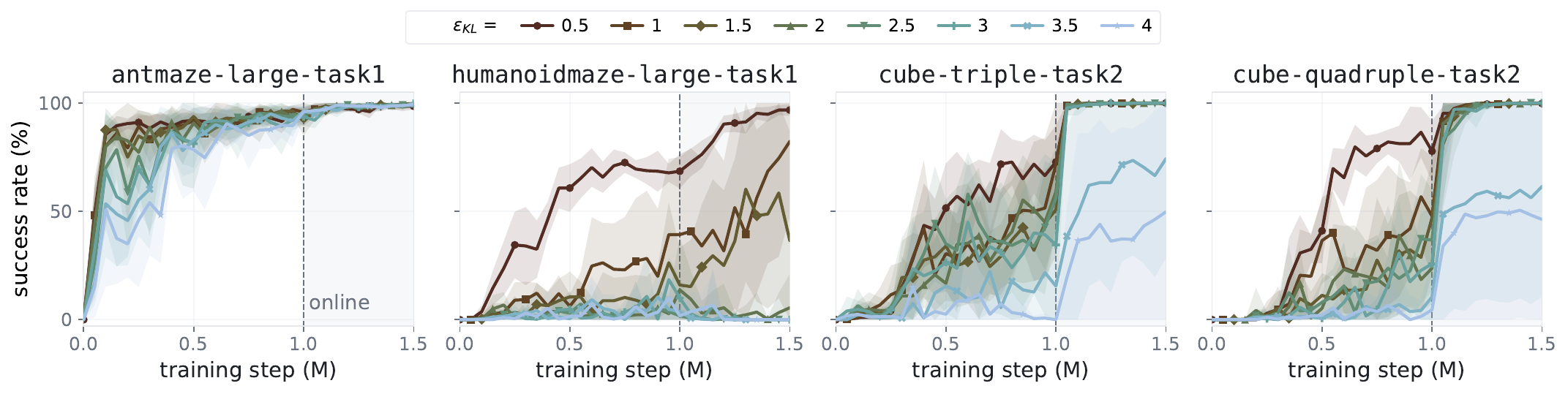}
\end{center}
\vspace{-0.1in}
\caption{Sweep over KL budget $\epsilon_{\text{KL}}$ on four of the ten swept domains. Eight seeds per
curve, and bands are one standard deviation. The full analysis on all ten domains is in
\cref{app:sensfull}.}
\label{fig:klsweep}
\vspace{-0.17in}
\end{figure}

\vspace{0.01in}
\paragraph{Effect of KL budget $\epsilon_{\text{KL}}$.}

The budget directly sets how far the fine-tuned policy may deviate from the pretrained one~\citep{dong2026trustregionqadjoint}. We sweep it on all ten domains (\cref{app:sensfull}), and \cref{fig:klsweep} shows four of them. On most domains, performance stays flat or decreases steadily as the budget widens, so the budget affects performance predictably and a small budget is a reliable default. The exception is
\texttt{puzzle-4x4}, which prefers a larger budget, consistent with its larger state space.

\subsection{Extension to Large Pretrained Policies}
\label{sec:scale}
To test whether \method{} extends to large pretrained policies, we fine-tune
RLDX-1~\citep{kim2026rldx1technicalreport}, a pretrained vision-language-action policy, on a real-world
bimanual robot~\citep{rainbow_rby1} equipped with two $20$-DoF dexterous hands~\citep{wuji_hand2},
with an action chunk of $16$ steps.
\cref{fig:task_vis} shows the three manipulation tasks. Starting from the pretrained policy, we
run supervised fine-tuning on task demonstrations, then \method{} for 15K offline steps and a further
10K online steps. We compare against
EXPO-FT~\citep{dong2026expoftsampleefficientreinforcementlearning}, a residual fine-tuning method
that freezes the pretrained policy and learns an edit policy that adds corrections to its actions
(\cref{app:baselines}), under the shared protocol of \cref{app:vlasetup}.
\cref{tab:vla} reports held-out success over $30$ episodes per task, and
\cref{fig:vlacurves} in \cref{app:rby1:results} tracks the training rollouts
through the online phase.

\method{} improves over SFT on all three tasks, whereas EXPO-FT does not appear to work well in
this setting. Whereas \method{} updates the flow policy directly, EXPO-FT trains a
residual edit policy, which must be learned from scratch and cannot directly exploit the
expressiveness of the pretrained flow policy. We conjecture that these two differences account for
this.

\begin{figure}[t]
\begin{center}
\includegraphics[width=\linewidth]{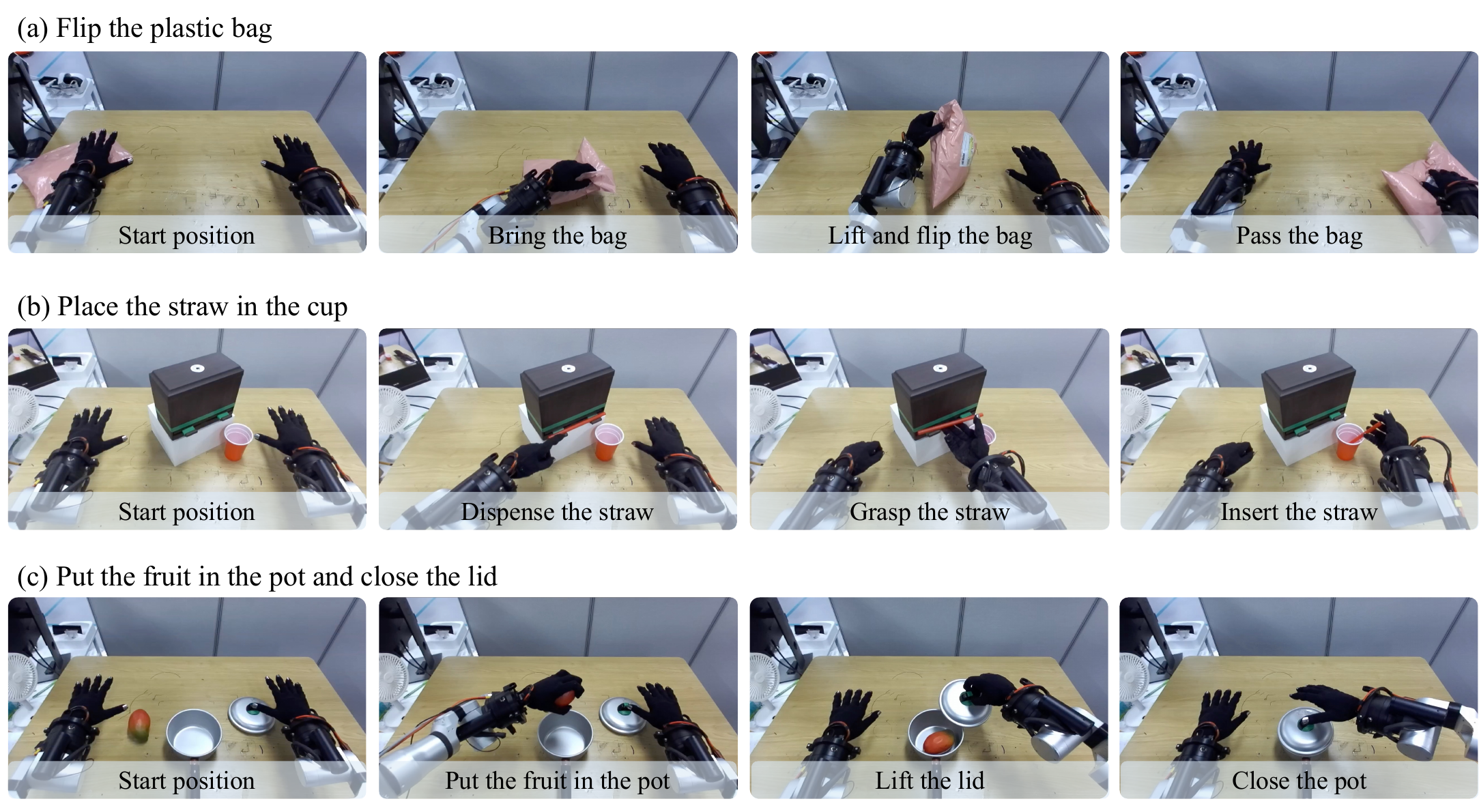}
\end{center}
\vspace{-0.1in}
\caption{Overview of the three real-world manipulation tasks. From top to bottom: flipping a plastic bag, placing a straw in a cup, and placing fruit in a pot and closing the lid. Each row illustrates the task progression from left to right.}
\label{fig:task_vis}

\end{figure}

\begin{table}[t] 
\centering
\caption{Success rate on three real-robot manipulation tasks with a vision-language-action policy. The
offline rows evaluate the checkpoint after 15K offline steps, and the offline-to-online rows the checkpoint
after a further 10K online steps. Each cell counts successes over $30$ held-out episodes.}
\label{tab:vla}
\begin{tabular}{lccc}
\toprule
Method & \makecell{Flip plastic bag} & \makecell{Place the straw} & \makecell{Put the fruit and close the lid} \\
\midrule
SFT & $15/30$ & $7/30$ & $4/30$ \\
EXPO-FT (offline) & $8/30$ & $0/30$ & $5/30$ \\
EXPO-FT (offline2online) & $15/30$ & $0/30$ & $3/30$ \\
\method{} (offline) & $20/30$ & $13/30$ & $8/30$ \\
\method{} (offline2online) & $\mathbf{25/30}$ & $\mathbf{15/30}$ & $\mathbf{13/30}$ \\
\bottomrule
\end{tabular}
\vspace{-0.15in}
\end{table}

\section{Related work}
\paragraph{Propagating value to intermediate flow steps.} A flow policy generates its action over many flow
steps, while the critic evaluates only the final action. Using the critic to steer the flow therefore requires propagating its value to each intermediate step, and existing methods differ mainly in how they do
so~\citep{lee2026rewardscorematchingunifying}. The most common approach evaluates the critic at a
Tweedie posterior estimate of the final action, and it is used for guidance at test
time~\citep{chung2024diffusionposteriorsamplinggeneral,kim2025flowdpsflowdrivenposteriorsampling,zhou2026testtimegradientguidanceflow}.
Another line avoids the approximation and feeds the intermediate action to the critic directly, either
by matching the policy's score to the critic's action gradient at the intermediate
action~\citep{psenka2025learningdiffusionmodelpolicy} or by learning a separate value at each noise
level~\citep{oberai2026reversalqlearning,doo2026qflowstableexpressivereinforcement}. Adjoint
matching~\citep{domingoenrich2025adjointmatchingfinetuningflow,li2026qlearningadjointmatching,dong2026trustregionqadjoint,bergmeister2026reinforceadjointmatching} gives a theoretically
grounded answer, propagating the value gradient back to each flow step through the lean adjoint ODE.
For test-time guidance rather than fine-tuning, dropping the Jacobian has also been found to outperform
backpropagating through the exact Jacobian~\citep{zhou2026testtimegradientguidanceflow}.

\paragraph{Critic regularization.} Many value learning schemes regularize the critic, and they differ
in where the regularization acts. CQL penalizes the critic's value at the actions the policy prefers relative to the data~\citep{kumar2020conservativeqlearningofflinereinforcement}, IQL fits a state value by expectile regression and trains the critic by TD regression toward it, using only dataset actions~\citep{kostrikov2021offlinereinforcementlearningimplicit}, and ReBRAC subtracts an action-distance penalty from the bootstrap target~\citep{tarasov2023revisitingminimalistapproachoffline}. Fisher-BRC penalizes the action
gradient of the critic's offset from a behavior model~\citep{kostrikov2021offlinereinforcementlearningfisher}, which is close to the gradient-norm penalty we compare in \cref{sec:gain}, and Cal-QL calibrates the critic against a reference policy for offline-to-online fine-tuning~\citep{nakamoto2024calqlcalibratedofflinerl}.

\paragraph{Offline-to-online RL.} Pretraining a policy and a critic on offline data and then fine-tuning
them online is a standard recipe for sample-efficient
RL~\citep{ball2023efficientonlinereinforcementlearning,hester2017deepqlearningdemonstrations,hu2024imitationbootstrappedreinforcementlearning,kostrikov2021offlinereinforcementlearningimplicit,kumar2020conservativeqlearningofflinereinforcement,lei2024unio4unifyingonlineoffline,luo2025serlsoftwaresuitesampleefficient,nair2020awac,nakamoto2024calqlcalibratedofflinerl,rajeswaran2018learningcomplexdexterousmanipulation,vecerik2018leveragingdemonstrationsdeepreinforcement}.
A central challenge is the distribution shift at the transition, which destabilizes the value
function and can make the policy forget what it learned
offline~\citep{ball2023efficientonlinereinforcementlearning,nakamoto2024calqlcalibratedofflinerl,wolczyk2024fine}.

\section{Conclusion}
We introduced \method{}, which fine-tunes flow policies with off-policy RL without a vector--Jacobian
product through the policy. Motivated by the diagonal structure of the batch-averaged velocity
Jacobian, it replaces the exact adjoint with a closed-form scalar adjoint and pairs it with a value
penalty at the policy-generated actions. \method{} improves most on the four hardest OGBench domains
and extends to a vision-language-action policy on a real robot.

Our results also suggest a broader point. A rough approximation of the adjoint suffices, while the
value learning scheme largely determines how well it performs. The bottleneck of
off-policy RL with flow policies may therefore lie less in the policy update than in the critic. We
hope this encourages further work on value learning for flow policies.
\subsubsection*{Reproducibility statement}
\cref{app:algorithm} gives the complete pseudocode of \method{}, and \cref{sec:decomp}
states and proves the closed form it builds on. \cref{app:expdetails} lists the domains,
datasets, and hyperparameters, and \cref{app:baselines} describes every baseline. The OGBench
experiments use the public benchmark~\citep{ogbench_park2025} and the datasets described in
\cref{app:expdetails}. \cref{app:vlasetup} documents the platform, observation and
action spaces, training details, and evaluation protocol of the real-robot experiments. We will release
the implementation of \method{}.

\subsubsection*{Ethics statement}
This work fine-tunes robot policies with reinforcement learning in simulation and on a robot in a
controlled laboratory setting. The real-robot experiments involve no human subjects and no personal
data, and an operator supervised the robot throughout. As with any method that improves a policy
against a learned value function, deployment beyond a supervised setting calls for the usual
safeguards.

\subsubsection*{AI use statement}
We used generative AI tools to assist with writing, editing, and routine coding and plotting. We did
not use them to generate research ideas, experimental results, or claims. The authors reviewed all
AI-assisted text and code and take full responsibility for the content of this paper.

\bibliographystyle{plainnat}
\bibliography{reference}

\clearpage
\appendix
\crefalias{section}{appendix}
\crefalias{subsection}{subappendix}
\crefalias{subsubsection}{subsubappendix}
\section*{Contents}
\startcontents[appendices]
\printcontents[appendices]{}{1}{\setcounter{tocdepth}{2}}

\newpage

\section{Full algorithm}
\label{app:algorithm}

\cref{alg_full:sqam} presents the complete pseudocode of \method{}, including the trust region it inherits from TRQAM~\citep{dong2026trustregionqadjoint}. \textcolor{GoogleBlue}{Blue} marks what \method{} changes from TRQAM.

The scalar adjoint also changes which states the policy loss is evaluated on. The exact adjoint
integrates the lean adjoint ODE in \cref{eq:adjoint} backward along a sampled trajectory,
\begin{align*}
    \dd \Aexact_\tau= -{\Aexact_\tau}^{\!\top}\nabla_{X_\tau}\Big(2v^{\mathrm{base}}(X_\tau,\tau)-\frac{1}{\tau}X_\tau\Big) \dd \tau,
    \qquad \Aexact_1=-\nabla_{X_1}Q^\pi(s,X_1),
\end{align*}
so the adjoint at each flow time depends on the later ones, and the loss states must be the states of
that trajectory. The scalar adjoint removes this coupling, so the loss states can be built from the
final action alone. Following \citet{bergmeister2026reinforceadjointmatching}, we sample the final
action with the ODE sampler instead of a long and numerically unstable stochastic rollout, and re-noise
it along the linear interpolation path, $X_\tau=(1-\tau)\,\epsilon+\tau\,a_\pi$ with fresh
$\epsilon\sim\mathcal N(0,I)$. We call these the bridge states, which match the rollout-state distribution at the optimum of the
stochastic optimal control problem and approximate it elsewhere~\citep{bergmeister2026reinforceadjointmatching}. In addition, the trust region is optional, and we use
it by default for stable training (\cref{app:notrust}).

\begin{algorithm}[H]
\caption{Q-learning with Scalar Adjoint Matching (\method{})}
\label{alg_full:sqam}
\begin{algorithmic}
\State \textbf{Input:} replay buffer $\mathcal{D}$;
$v^{\mathrm{base}}$: pretrained velocity field;
$v^{\mathrm{ft}}_\theta$: fine-tuned velocity field;
$Q_\phi$: critic function;
step size $h$;
KL budget $\varepsilon_{\mathrm{KL}}$;
dual step size $\eta_\lambda$;
EMA coefficient $\rho$;
\textcolor{GoogleBlue}{value penalty coefficient $c$};
fine-tuning iterations $N$.
\State \textbf{Initialize:}
$v^{\mathrm{ft}}_\theta \gets v^{\mathrm{base}}$ with parameters $\theta$;
$\lambda_0 > 0$;
$\overline{D}_0 \gets 0$.
\State \textbf{Noise schedule:}
$\sigma_n(\tau) := g(\tau)/\sqrt{\lambda_n}$, where $g(\tau) := \sqrt{2(1-\tau)/\tau}$
\vspace{2pt}
\For{$n \in \{0, \dots, N-1\}$}
    \State Sample a batch $\mathcal{B} = \{(s_i, a_i, r_i, s'_i)\}$ from $\mathcal{D}$
    \vspace{2pt}
    \State \textbf{Rollout:}
    For each state $s \in \mathcal{B}$, sample the endpoint $X_1$ \textcolor{GoogleBlue}{by the flow
    itself}:
    \begin{equation*}
        \textcolor{GoogleBlue}{X_{\tau + h}
        = X_\tau
        + h\, v^{\mathrm{ft}}_\theta(s, X_\tau, \tau),
        \quad
        X_0 \sim \mathcal{N}(0, I),}
    \end{equation*}
    and set $a_\pi \gets \mathtt{stopgrad}(X_1)$.
    \State \textcolor{GoogleBlue}{\textbf{Bridge states:}
    $X_\tau \gets (1-\tau)\,\epsilon + \tau\,a_\pi$ with fresh $\epsilon \sim \mathcal{N}(0, I)$,
    for all $\tau \in \{h, 2h, \dots, 1\}$}
    \State \textcolor{GoogleBlue}{\textbf{Scalar adjoint:}
    $\Ahat_\tau \gets -\tau\,\nabla_{a} Q_\phi(s, a_\pi)$ for all $\tau \in \{h, 2h, \dots, 1\}$}
    \State Stop gradient: $X_\tau \gets \mathtt{stopgrad}(X_\tau)$,
    $\Ahat_\tau \gets \mathtt{stopgrad}(\Ahat_\tau)$.
    \vspace{2pt}
    \State \textbf{Critic update:}
    Optimize $\phi$ w.r.t.
    \begin{equation*}
        \mathcal{L}(\phi)
        = \frac{1}{|\mathcal{B}|} \sum_{(s,a,r,s') \in \mathcal{B}}
          \Big[
            \big(Q_\phi(s, a) - r - \gamma\, Q_{\bar{\phi}}(s',\, a' \sim \pi_\theta(\cdot \mid s'))\big)^2
            \;+\; \textcolor{GoogleBlue}{c\,\big(Q_\phi(s, a_\pi) - Q_\phi(s, a)\big)}
          \Big]
    \end{equation*}
    \vspace{2pt}
    \State \textbf{Policy update:}
    Optimize $\theta$ w.r.t. the adjoint matching objective:
    \begin{equation*}
        \mathcal{L}_{\mathrm{Adj\text{-}Match}}(\theta)
        = \frac{1}{|\mathcal{B}|}\!\sum_{s \in \mathcal{B}}\sum_{\tau}
          \left\|
              \tfrac{2}{\sigma_n(\tau)}
              \big(v^{\mathrm{ft}}_\theta(s, X_\tau, \tau)
                   - v^{\mathrm{base}}(s, X_\tau, \tau)\big)
              + \sigma_n(\tau)\,\textcolor{GoogleBlue}{\Ahat_\tau}
          \right\|^2
    \end{equation*}
    \vspace{2pt}
    \State \textbf{Trust region update:}
    \State Estimate path-space KL surrogate:
    \begin{equation*}
        \widehat{D}_n
        = \frac{1}{|\mathcal{B}|}\!\sum_{s \in \mathcal{B}}\sum_{\tau}
          \tfrac{2h}{g(\tau)^2}
          \big\|v^{\mathrm{ft}}_\theta(s, X_\tau, \tau) - v^{\mathrm{base}}(s, X_\tau, \tau)\big\|^2
    \end{equation*}
    \State EMA smoothing:
    $\overline{D}_n \gets (1-\rho)\,\overline{D}_{n-1} + \rho\,\widehat{D}_n$.
    \State Dual descent:
    $\lambda_{n+1} \gets \max\!\big\{0,\ \lambda_n + \eta_\lambda(\overline{D}_n - \varepsilon_{\mathrm{KL}})\big\}$.
\EndFor
\State \textbf{Output:} fine-tuned velocity field $v^{\mathrm{ft}}_\theta$, critic $Q_\phi$.
\end{algorithmic}
\end{algorithm}

\section{Baselines}
\label{app:baselines}

On OGBench, we compare against seven off-policy fine-tuning methods for flow policies, spanning the
main fine-tuning paradigms. On the real robot, we compare against EXPO-FT (\cref{sec:scale}). All of them read transitions $(s,a,r,s')$ uniformly from a replay buffer
$\mathcal{D}$, which holds the offline dataset during offline training and grows with the collected
rollouts online.

\paragraph{FQL~\citep{park2025flowqlearning}.} FQL avoids differentiating through the sampling chain by
distilling the multi-step flow into a one-step policy. The behavior-cloning rollout follows the ODE
\begin{align*}
    \dd X_\tau= v_\theta(X_\tau, \tau; s)\,\dd\tau, \quad X_0 \sim \Normal(0, I), \quad \tau \in [0, 1],
\end{align*}
and we denote its terminal value by $\mathrm{ODE}(v_\theta, s, X_0) := X_1$. A one-step policy
$\pi_\omega(s, X_0)$ trains jointly with $v_\theta$, maximizing the critic while an $\ell_2$ term holds
it near the rollout,
\begin{align*}
    \mathcal L_{\mathrm{onestep}}(\omega) = \E_{X_0 \sim \Normal}\!\Big[
    \underbrace{-\, Q(s, \pi_\omega(s, X_0))}_{\text{RL maximization}}
    \;+\;
    \underbrace{\alpha\, \|\pi_\omega(s, X_0) - \mathrm{ODE}(v_\theta, s, X_0)\|_2^2}_{\text{BC distillation}}
    \Big].
\end{align*}
The coefficient $\alpha$ sets how closely $\pi_\omega$ follows the distillation target, and
$\pi_\omega$ is the policy that acts in the environment.

\paragraph{CGQL-L~\citep{dhariwal2021diffusion}.} CGQL-Linex is the guidance baseline introduced in~\citet{li2026qlearningadjointmatching}. It combines the BC
velocity field with classifier-free guidance~\citep{dhariwal2021diffusion} from a Q-function. An
auxiliary critic $Q_\psi(s, X_\tau, \tau)$ is trained on intermediate noisy actions and induces the
guidance velocity
\begin{align*}
    \hat{v}_\psi(X_\tau, \tau; s) := \frac{(1 - \tau)\beta\,\nabla_{X_\tau} Q_\psi(s, X_\tau, \tau) + X_\tau}{\tau},
\end{align*}
where $X_\tau = (1-\tau) X_0 + \tau X_1$ with $X_0 \sim \Normal(0, I)$. Sampling runs the combined
velocity $v = v^{\mathrm{base}} + w\, \hat{v}_\psi$ with guidance weight $w$, steering toward the
entropy-regularized target $\pi^\star(\cdot \mid s) \propto e^{\beta Q_\phi(s, \cdot)}$. The auxiliary
critic trains with a Linex regression~\citep{parsian2002estimation,myers2025offline},
\begin{align*}
    \mathcal L_{\mathrm{Linex}}(\psi) = \E_{\tau, X_0}\!\left[
    \exp\bigl(\beta(Q_\phi(s, X_1) - Q_\psi(s, X_\tau, \tau))\bigr) + \beta Q_\psi(s, X_\tau, \tau)\right],
\end{align*}
and the standard critic $Q_\phi$ trains by TD with target actions sampled from the combined velocity,
\begin{align*}
    \mathcal L_{\mathrm{TD}}(\phi) = \E\!\left[
    \bigl(Q_\phi(s, a) - r - \gamma\, Q_{\bar\phi}(s', \mathrm{ODE}(v, s', X_0))\bigr)^2
    \right].
\end{align*}
Following~\citet{li2026qlearningadjointmatching}, the Linex loss carries a Huber-style stabilization against exponential blow-up.

\paragraph{DSRL~\citep{wagenmaker2025steeringdiffusionpolicylatent}.} DSRL leaves the flow frozen and runs RL in its noise
space. A one-step Gaussian noise-space policy $\pi_\omega(X_0 \mid s)$ trains with SAC against a
noise-space critic $Q_\psi(s, X_0)$ that regresses to the action-space critic,
\begin{align*}
    \mathcal L(\psi) = \E_{X_0 \sim \Normal}\!\left[
    \bigl(Q_\psi(s, X_0) - Q_\phi(s, \mathrm{ODE}(v_{\bar\theta}, s, X_0))\bigr)^2
    \right].
\end{align*}
At inference the noise is drawn from the noise-space policy, $X_0\sim\pi_\omega(\cdot\mid s)$, and
pushed through the flow, $a = \mathrm{ODE}(v_\theta, s, X_0)$.
Following~\citet{li2026qlearningadjointmatching}, our DSRL also fine-tunes the BC velocity online, with a target network
$v_{\bar\theta}$ for stability, which strengthens its offline-to-online performance.

\paragraph{IFQL.} IFQL is the flow counterpart of implicit diffusion
Q-learning~\citep{hansen2023idql}, taken as a baseline in~\citet{park2025flowqlearning}. Value learning is
IQL-style expectile regression~\citep{kostrikov2021offlinereinforcementlearningimplicit}. A value network $V_\xi$ fits an upper
expectile of the critic, and the critic bootstraps through it,
\begin{align*}
    \mathcal L_V(\xi) &= \E\!\left[L_2^\kappa\bigl(Q_{\bar\phi}(s, a) - V_\xi(s)\bigr)\right], \\
    \mathcal L_Q(\phi) &= \E\!\left[\bigl(r + \gamma V_\xi(s') - Q_\phi(s, a)\bigr)^2\right],
\end{align*}
where $L_2^\kappa(u) = |\kappa - \mathbf{1}(u < 0)|\, u^2$ is the expectile loss with
$\kappa \in (0.5, 1)$. The policy draws $N$ candidate actions from the BC flow,
$a^{(i)} = \mathrm{ODE}(v_\theta, s, X_0^{(i)})$, and acts with the one the critic values highest.

\paragraph{QAM and QAM-E~\citep{li2026qlearningadjointmatching}.} Adjoint matching casts fine-tuning as stochastic optimal
control, which adds a control $u$ to the drift $b$ of the sampler,
\begin{align}
    \dd X_\tau^u = \big(b(X_\tau^u,\tau)+\sigma(\tau)u(X_\tau^u,\tau)\big)\dd\tau
        +\sqrt{\lambda}\,\sigma(\tau)\,\dd B_\tau ,
    \label{equ:controlled_sde}
\end{align}
with $B_\tau$ a Brownian motion and $\lambda$ a scale on the diffusion, and the base sampler
is \cref{equ:controlled_sde} at $u=0$. QAM solves this control problem at a constant $\lambda$, with
an inverse temperature $\beta$ on the terminal reward $\beta\, Q_\phi(s, X_1)$. The fine-tuned velocity trains against the base velocity through the
lean adjoint matching loss
\begin{align*}
    \mathcal L_{\mathrm{AM}}(\theta) = \E\!\left[\int_0^1 \left\|\frac{2\bigl(v^{\mathrm{ft}}_\theta(X_\tau, \tau; s) - v^{\mathrm{base}}(X_\tau, \tau; s)\bigr)}{\sigma(\tau)} + \sigma(\tau)\, \Aexact_\tau\right\|_2^2 \dd\tau\right],
\end{align*}
where $\Aexact_\tau$ is the lean adjoint with terminal condition
$\Aexact_1 = -\beta\, \nabla_{X_1} Q_\phi(s, X_1)$ and $X_\tau$ follows the sampling
dynamics of \cref{equ:controlled_sde} at that constant $\lambda$. Element-wise gradient clipping keeps
the update numerically stable. QAM-E adds a residual edit policy $\pi_\omega(\Delta a \mid s, \tilde a)$
on top, perturbing the QAM-generated action $\tilde a$ by at most $\sigma_a$ in $L_\infty$ distance
through a tanh-squashed Gaussian, trained with entropy-regularized SAC and automatic entropy tuning.

\paragraph{TRQAM~\citep{dong2026trustregionqadjoint}.} TRQAM is the trust-region member of the same family and the
strongest baseline. Where QAM holds the temperature fixed, TRQAM reads $\lambda$
in \cref{equ:controlled_sde} as a trust-region parameter on the sampling dynamics and adapts it over
training. Girsanov's theorem gives the path-space KL between the controlled and the base dynamics in
closed form,
\begin{align}
    D_{\mathrm{KL}}\big(\mathbb P^{u}\,\|\,\mathbb P^{\mathrm{base}}\big)
    \;=\;\E_{\mathbf X\sim\mathbb P^u}\bigg[\int_0^1
    \frac{2}{\lambda\,\sigma(\tau)^2}\,\big\|v^{\mathrm{ft}}_\theta(X^u_\tau,\tau)-v^{\mathrm{base}}(X^u_\tau,\tau)\big\|^2\,\dd\tau\bigg] ,
    \label{eq:pathkl}
\end{align}
so the deviation can be measured on every rollout at no additional cost.
Each update estimates the path-space KL of \cref{eq:pathkl} on its own rollout, smooths the estimate with an EMA
$\overline D_n$, and takes the projected dual step
\begin{align*}
    \lambda_{n+1}=\max\{0,\;\lambda_n+\eta_\lambda(\overline D_n-\varepsilon_{\mathrm{KL}})\}
\end{align*}
with step size $\eta_\lambda$. The constraint therefore lives inside the sampling dynamics, held at the
prescribed budget $\varepsilon_{\mathrm{KL}}$, rather than as a loss-level penalty that a strong critic
signal can override.

\paragraph{EXPO and EXPO-FT~\citep{dong2026expostablereinforcementlearning,dong2026expoftsampleefficientreinforcementlearning}.}
EXPO trains the expressive base policy $\pi_{\mathrm{base}}$ with a stable imitation objective and
never maximizes the value through it. A light-weight Gaussian edit policy $\pi_\omega(\Delta\mid s,a)$
instead adds a residual to the base actions, $\tilde a = a + \beta\Delta$ with edit scale $\beta$, and
is trained with an entropy-regularized objective to raise $Q_\phi(s,\tilde a)$. Actions are chosen by an
on-the-fly policy that draws $N$ base actions $a_{1:N}\sim\pi_{\mathrm{base}}(\cdot\mid s)$, edits each
of them, and takes the candidate with the highest value,
\begin{align*}
    a^\star = \argmax_{a\in\{a_i,\tilde a_i\}_{i=1}^N} Q_\phi(s,a),
\end{align*}
both for acting and for the bootstrap action in the TD target. EXPO-FT applies this recipe to a
pretrained vision-language-action policy~\citep{dong2026expoftsampleefficientreinforcementlearning}. We use EXPO-FT as the baseline
of \cref{sec:scale}, with the settings of \cref{app:rby1:training}.

\section{Value learning schemes}
\label{app:consforms}

All schemes compared in \cref{fig:conservatism} run inside the same \method{} implementation
and differ only in how the critic is trained. The default critic loss is the temporal-difference loss
\begin{align}
    \mathcal{L}_{\mathrm{TD}}(\phi)
    = \E_{(s,a,r,s')\sim\mathcal{D}}\!\left[
      \big(Q^\pi_\phi(s, a) - r - \gamma\, Q^\pi_{\bar\phi}(s', a')\big)^2
    \right],
    \qquad a' \sim \pi_\theta(\cdot \mid s'),
    \label{eq:app_td}
\end{align}
where $a_\pi$ denotes the final action the actor update has already drawn from $\pi_\theta(\cdot\mid s)$.
The value penalty and the gradient-norm penalty add a term to \cref{eq:app_td}, ReBRAC modifies its
bootstrap target, and the in-sample critic replaces it with a value target fit by expectile regression.
We summarize each scheme below.

\paragraph{Value penalty.} Adds to \cref{eq:app_td} the gap between the critic's value at the policy's
own action and at the dataset action, as in \cref{eq:gapreg},
\begin{align*}
    \mathcal{L}(\phi)
    = \mathcal{L}_{\mathrm{TD}}(\phi)
    + c\,\E_{s\sim\mathcal{D}}\!\left[
      Q^\pi_\phi(s, \operatorname{sg}(a_\pi)) - Q^\pi_\phi(s, a_{\mathrm{data}})
    \right].
\end{align*}
The stop-gradient keeps the penalty's gradient off the policy, and reusing $a_\pi$ means the
penalty draws no additional sample (\cref{sec:critic}).

\paragraph{Gradient-norm penalty.} Constrains the magnitude of the gradient the scalar adjoint uses rather
than the value penalty,
\begin{align*}
    \mathcal{L}(\phi)
    = \mathcal{L}_{\mathrm{TD}}(\phi)
    + c_g\,\E_{s\sim\mathcal{D}}\!\left[
      \big\|\nabla_a Q^\pi_\phi(s, \operatorname{sg}(a_\pi))\big\|^2
    \right].
\end{align*}
A first-order expansion of the value penalty relates the two penalties (\cref{sec:gain}). 

\paragraph{In-sample critic~\citep{kostrikov2021offlinereinforcementlearningimplicit}.} Replaces the
bootstrap action with a state value fit by expectile regression, so the critic is never evaluated off
the dataset's actions,
\begin{align*}
    \mathcal{L}_V(\xi) &= \E_{(s,a)\sim\mathcal{D}}\!\left[L_2^\kappa\big(Q^\pi_{\bar\phi}(s, a) - V_\xi(s)\big)\right], \\
    \mathcal{L}(\phi) &= \E_{(s,a,r,s')\sim\mathcal{D}}\!\left[\big(Q^\pi_\phi(s, a) - r - \gamma\, V_\xi(s')\big)^2\right],
\end{align*}
with $L_2^\kappa(u) = |\kappa - \mathbf{1}(u < 0)|\, u^2$.

\paragraph{ReBRAC's bootstrap penalty~\citep{tarasov2023revisitingminimalistapproachoffline}.}
Subtracts the distance between the policy's next-state action and the dataset's from the bootstrap
target, so the control acts on the actions rather than on the value the critic assigns them,
\begin{align*}
    \mathcal{L}(\phi)
    = \E_{(s,a,r,s')\sim\mathcal{D}}\!\left[
      \Big(Q^\pi_\phi(s, a) - r - \gamma\big(Q^\pi_{\bar\phi}(s', a') - \beta\lVert a' - a'_{\mathrm{data}}\rVert^2\big)\Big)^2
    \right],
\end{align*}
with $a' \sim \pi_\theta(\cdot\mid s')$ and $a'_{\mathrm{data}}$ the dataset action at $s'$.

\section{Experimental details}
\label{app:expdetails}

\subsection{Domains and tasks}
\label{app:domains}
We evaluate on OGBench~\citep{ogbench_park2025}, an offline goal-conditioned RL benchmark, using its
reward-based single-task variants. We use all $10$ domains, spanning long-horizon navigation,
multi-object manipulation, and combinatorial planning, with $5$ tasks per domain and $50$ tasks in
all. The domains are abbreviated as \texttt{scene}, \texttt{puzzle-3x3} (\texttt{p33}),
\texttt{puzzle-4x4} (\texttt{p44}), \texttt{cube-double} (\texttt{c2}), \texttt{cube-triple}
(\texttt{c3}), \texttt{cube-quadruple} (\texttt{c4}), \texttt{humanoidmaze-medium} (\texttt{hm}),
\texttt{humanoidmaze-large} (\texttt{hl}), \texttt{antmaze-large} (\texttt{al}),
and \texttt{antmaze-giant} (\texttt{ag}).

The dataset size, episode length, and action dimension of each domain are reported in
\cref{tab:metadata}. For each method and task we run $8$ random seeds. Unless otherwise stated,
tables report mean success rate $\pm$ one standard deviation across seeds, and plots show the mean with
shaded regions denoting the same.

\paragraph{Dataset structure.} Unlike teleoperation-style benchmarks where each demonstration solves
the target task directly, OGBench's offline data is task-agnostic. The \texttt{navigate} datasets
capture free maze exploration and the \texttt{play} datasets capture unstructured object manipulation,
so the behavior-cloned policy is a behavioural prior rather than a task-specific solution.

\paragraph{Dataset sources.} We use the official OGBench datasets~\citep{ogbench_park2025} except where
noted, and follow the dataset choices of \citet{dong2026trustregionqadjoint} for the three larger ones. For
\texttt{cube-triple-10M-*} and \texttt{puzzle-4x4-10M-*} we use a 10M subset of the official 100M
release, which is split into 100 files of 1M transitions each, taking the first 10 files sorted by
name~\citep{kim2025deasdetachedvaluelearning,dong2026trustregionqadjoint}. For \texttt{antmaze-giant-10M-*}, OGBench does not release a
dataset at this size, so we use the dataset of \citet{dong2026trustregionqadjoint}, which was generated with the
official OGBench data-generation pipeline at default settings.

\subsection{Hyperparameters}
\label{app:hyperparams}
Most methods share the common hyperparameters in \cref{tab:rl-hyperparams}, and the
domain-specific ones are listed in \cref{tab:hparam}. The baseline results are drawn from the
exported runs of \citet{dong2026trustregionqadjoint}, which adopt the per-domain values reported as optimal
by \citet{li2026qlearningadjointmatching} for FQL, DSRL, IFQL, CGQL-L, QAM, and QAM-E, and tune $\varepsilon_{\mathrm{KL}}$ for
TRQAM. \method{} adopts the per-domain KL budgets of TRQAM, including the time-varying schedule on
\texttt{antmaze-giant}, $\varepsilon_{\mathrm{KL}}=0.5$ offline opening to $3.0$ at the online
transition. We tune the value penalty coefficient following the protocol of \citet{dong2026trustregionqadjoint},
comparing $c\in\{0, 0.3\}$ on two tuning tasks per domain, task 1 and task 4 for locomotion and task 2
and task 4 for manipulation, with 2 seeds per configuration that differ from the evaluation seeds. On
the two domains where TRQAM uses $\varepsilon_{\mathrm{KL}}=1.0$, \texttt{cube-quadruple} and
\texttt{antmaze-large}, we additionally check $\varepsilon_{\mathrm{KL}}=0.5$ under the same protocol, and adopt it on
\texttt{cube-quadruple}. The selected values are listed in \cref{tab:hparam}, and all reported
results average over $8$ evaluation seeds.

\begin{table}[ht]
\centering
\small
\caption{\textbf{Domain metadata.} Following OGBench, we group the domains into locomotion and
manipulation. Action dimensions are per control step, and the manipulation domains use action chunks
of size $5$.}
\label{tab:metadata}
\setlength{\tabcolsep}{6pt}
\renewcommand{\arraystretch}{0.95}
\begin{tabular}{@{}llccc@{}}
\toprule
\textbf{Domain} & \textbf{Category} & \textbf{Data} & \textbf{Horizon} & \textbf{Act. dim.} \\
\midrule
\texttt{cube-double-*}          & Manipulation & 1M   & 500  & 5  \\
\texttt{cube-triple-10M-*}      & Manipulation & 10M  & 1000 & 5  \\
\texttt{cube-quadruple-100M-*}  & Manipulation & 100M & 1000 & 5  \\
\texttt{antmaze-large-*}        & Locomotion & 1M   & 1000 & 8  \\
\texttt{antmaze-giant-10M-*}    & Locomotion & 10M  & 1000 & 8  \\
\texttt{humanoidmaze-medium-*}  & Locomotion & 1M   & 2000 & 21 \\
\texttt{humanoidmaze-large-*}   & Locomotion & 1M   & 2000 & 21 \\
\texttt{scene-*}                & Manipulation & 1M   & 750  & 5  \\
\texttt{puzzle-3x3-*}           & Manipulation & 1M   & 500  & 5  \\
\texttt{puzzle-4x4-10M-*}       & Manipulation & 10M  & 500  & 5  \\
\bottomrule
\end{tabular}
\end{table}

\begin{table}[!htbp]
\centering
\small
\caption{\textbf{Common hyperparameters.}}
\label{tab:rl-hyperparams}
\setlength{\tabcolsep}{5pt}
\renewcommand{\arraystretch}{0.95}
\begin{tabular}{@{}lc@{}}
\toprule
\textbf{Parameter} & \textbf{Value} \\
\midrule
Batch size                              & 256 \\
Discount factor ($\gamma$)              & 0.995 (default), 0.999 (\texttt{humanoidmaze}) \\
Optimizer                               & Adam \\
Learning rate                           & $3 \times 10^{-4}$ \\
Target network update rate              & $5 \times 10^{-3}$ \\
Critic ensemble size ($K$)              & 10 \\
Critic pessimism coefficient ($\rho$)   & 0.5 (default), 0 (\texttt{humanoidmaze}) \\
UTD ratio                               & 1 \\
Number of flow steps ($T$)              & 10 \\
BC training steps                       & $0.3 \times 10^{6}$ \\
Offline RL steps                        & $10^{6}$ \\
Online RL steps                         & $0.5 \times 10^{6}$ \\
Network width                           & 512 (default), 1024 (10M/100M data) \\
Network depth                           & 4 \\
Gradient max-norm clipping              & False (default), 1 (\texttt{QAM}, \texttt{QAM-E}, \texttt{TRQAM}, \texttt{SQAM}) \\
Actor layer norm                        & False (default), True (10M/100M data) \\
Critic layer norm                       & True \\
\bottomrule
\end{tabular}
\end{table}

\begin{table}[!htbp]
\centering
\small
\caption{\textbf{Domain-specific hyperparameters.}}
\label{tab:hparam}
\setlength{\tabcolsep}{3pt}
\renewcommand{\arraystretch}{1.03}
\begin{tabular}{@{}lcccccccc@{}}
\toprule
Domain
  & \texttt{FQL}
  & \texttt{DSRL}
  & \texttt{IFQL}
  & \texttt{CGQL-L}
  & \texttt{QAM}
  & \texttt{QAM-E}
  & \texttt{TRQAM}
  & \method{} \\
  & $\alpha$
  & $\sigma_z$
  & $\kappa$
  & $(\vartheta,\varrho,\tau)$
  & $\beta$
  & $(\beta,\sigma_a)$
  & $\varepsilon_{\mathrm{KL}}$
  & $(\varepsilon_{\mathrm{KL}}, c)$ \\
\midrule
\texttt{scene-*}                & 300 & 0.4 & 0.9  & $(10, 0.1, 0.1)$   & 1  & $(1, 0)$    & 0.5 & $(0.5, 0.3)$ \\
\texttt{puzzle-3x3-*}           & 300 & 1.0 & 0.95 & $(10, 0.001, 0.1)$ & 3  & $(1, 0.1)$  & 2.0 & $(2.0, 0.3)$ \\
\texttt{puzzle-4x4-10M-*}       & 1   & 1.0 & 0.9  & $(10, 0.001, 1)$   & 30 & $(0.1, 0.9)$& 4.0 & $(4.0, 0.3)$ \\
\texttt{cube-double-*}          & 300 & 1.0 & 0.9  & $(10, 0.001, 0.01)$& 1  & $(1, 0)$    & 0.5 & $(0.5, 0)$ \\
\texttt{cube-triple-10M-*}      & 30  & 1.4 & 0.95 & $(10, 0.001, 0.1)$ & 3  & $(3, 0.1)$  & 0.5 & $(0.5, 0.3)$ \\
\texttt{cube-quadruple-100M-*}  & 100 & 1.4 & 0.95 & $(10, 0.1, 0.01)$  & 1  & $(3, 0.1)$  & 1.0 & $(0.5, 0.3)$ \\
\texttt{antmaze-large-*}        & 3   & 0.8 & 0.9  & $(10, 0.001, 0.1)$ & 10 & $(1, 0.1)$  & 1.0 & $(1.0, 0.3)$ \\
\texttt{antmaze-giant-10M-*}    & 3   & 1.2 & 0.8  & $(10, 0.001, 0.1)$ & 3  & $(10, 0.1)$ & $0.5\!\to\!3.0$ & $(0.5\!\to\!3.0, 0.3)$ \\
\texttt{humanoidmaze-medium-*}  & 30  & 0.6 & 0.7  & $(10, 0.1, 0.1)$   & 3  & $(3, 0.1)$  & 0.5 & $(0.5, 0.3)$ \\
\texttt{humanoidmaze-large-*}   & 30  & 0.8 & 0.8  & $(10, 0.1, 0.1)$   & 3  & $(3, 0.1)$  & 0.5 & $(0.5, 0.3)$ \\
\bottomrule
\end{tabular}
\end{table}

\section{Real-robot experiments: RB-Y1 with Wuji Hand~2}
\label{app:vlasetup}

This appendix documents the real-robot experiments on a dual-arm RB-Y1
humanoid~\citep{rainbow_rby1} equipped with two Wuji Hand~2 dexterous hands~\citep{wuji_hand2}. For
each of three manipulation tasks we follow the same pipeline: (i)~behavior cloning (BC) of a
vision-language-action base policy on teleoperated demonstrations, (ii)~collection of 30 rollout
episodes of the BC policy on the robot, (iii)~offline RL for 15K steps on the demonstrations plus these
30 BC rollout episodes, (iv)~online RL for a further 10K steps that continues from the offline
checkpoint while the robot collects new episodes, and (v)~a 30-episode evaluation of every checkpoint.
The same protocol is run for our method (\method{}) and for the EXPO-FT
baseline~\citep{dong2026expoftsampleefficientreinforcementlearning}.

\subsection{Platform, observations and actions}
\label{app:rby1:platform}
\begin{figure}
    \centering
    \includegraphics[width=0.5\linewidth]{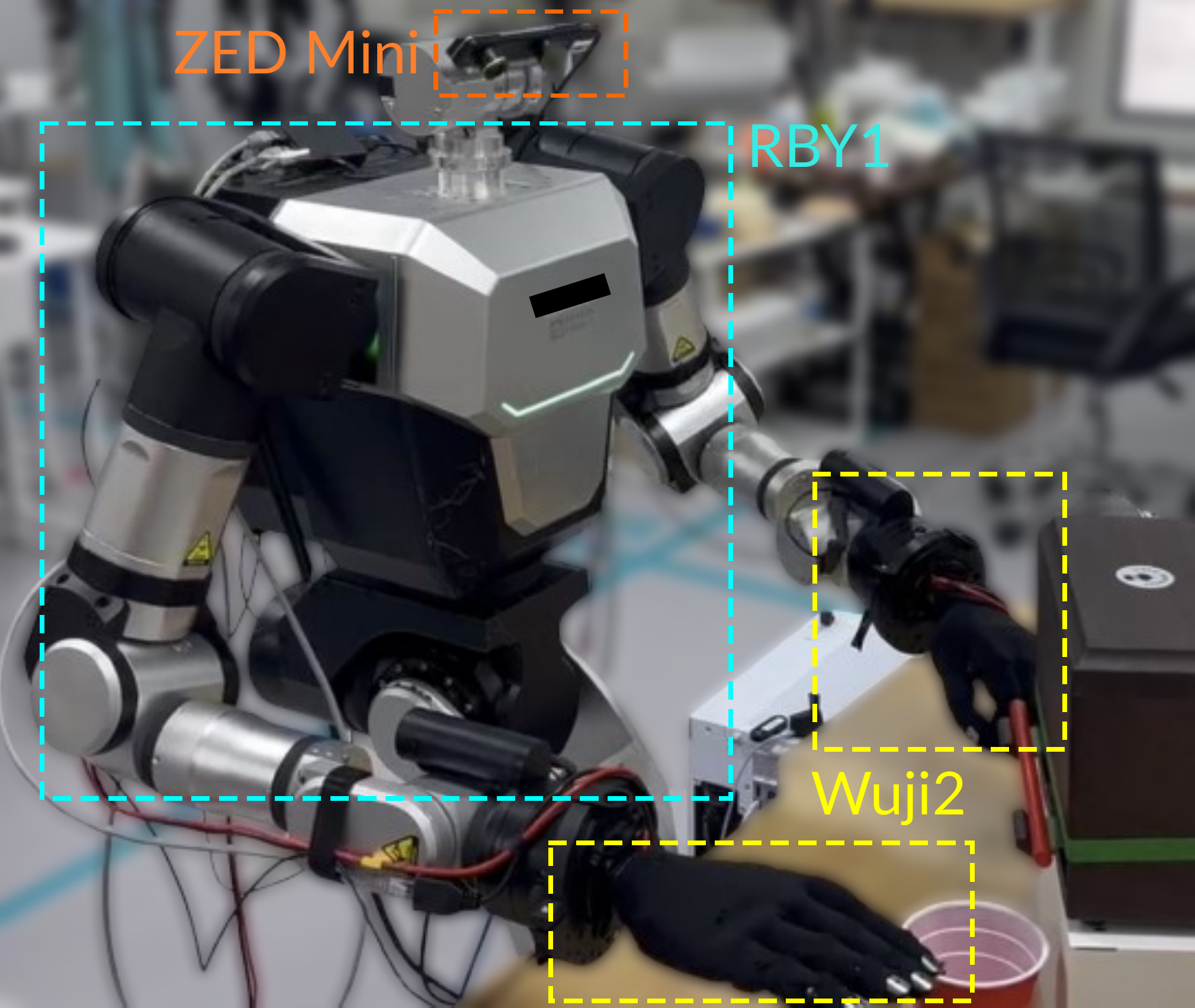}
    \caption{Hardware setup: RB-Y1 humanoid with a Wuji Hand~2 on each wrist
and a head-mounted ZED Mini stereo camera.}
    \label{fig:rby1:hw}
\end{figure}

\paragraph{Hardware.}
We use a Rainbow Robotics RB-Y1~\citep{rainbow_rby1} dual-arm humanoid with a 20-DoF Wuji
Hand~2~\citep{wuji_hand2} mounted on each wrist. The experiments use 54 of its degrees of freedom, the two 7-DoF arms and the
two 20-DoF hands, for bimanual tabletop manipulation. The torso and neck stay at fixed joint positions, and
the mobile base is not used. A single Stereolabs ZED Mini stereo camera~\citep{stereolabs_zedmini} mounted on the
head provides the visual input (\cref{fig:rby1:hw}).

\paragraph{Observation space.}
At every policy query the observation has three parts. The first is two ego-centric RGB images, the
left and right views of the ZED Mini stereo pair, captured at 30\,Hz and resized to $192\times320$. The
second is a 62-dimensional proprioceptive vector of measured joint positions. The third is the task instruction
(\cref{tab:rby1:tasks}).

\paragraph{Action space.}
Actions are absolute joint-position targets. The base policy predicts an action chunk of 24 steps, and
its first 16 steps are executed. RL acts only on these $16$ executed steps and on the 54 arm and
hand joints. \method{} optimizes this sub-chunk by
updating the action head of the base policy, whereas EXPO-FT keeps the base policy frozen and adds a
residual to it through its edit policy.

\paragraph{Control loop.}
The robot runs at a nominal control rate of 30\,Hz and executes the policy synchronously. At each
policy query the current observation is sent to the actor and the returned 16 step chunk of
joint-position targets is executed step by step. A new observation is taken and the next chunk is
requested only after the chunk has been completed.

\subsection{Tasks and evaluation protocol}
\label{app:rby1:eval}

\cref{tab:rby1:tasks} lists the three tasks with their language instructions,
demonstration counts, and horizon parameters. All three tasks use a sparse reward, which is $1$ when the
episode succeeds and $0$ otherwise.

\begin{table}[!htbp]
\centering
\small
\caption{\textbf{Task specification.} For each task we list the language instruction, the number of
teleoperated demonstrations, and the discount factor $\gamma$.}
\label{tab:rby1:tasks}
\setlength{\tabcolsep}{5pt}
\renewcommand{\arraystretch}{0.95}
\begin{tabular}{@{}llcc@{}}
\toprule
\textbf{Task} & \textbf{Language instruction} & \textbf{Demos} & $\gamma$ \\
\midrule
\texttt{flip-plastic-bag} & ``flip the plastic bag''                       & 50 & 0.999 \\
\texttt{place-straw}      & ``place the straw in the cup''                 & 50 & 0.999 \\
\texttt{put-and-close}    & ``put the fruit in the pot and close the lid'' & 50 & 0.999 \\
\bottomrule
\end{tabular}
\end{table}

Each task uses a fixed initial configuration for the teleoperated demonstrations, training, and
evaluation, identical across methods. Every checkpoint is
evaluated over $30$ held-out episodes, and success is judged by the operator from the final state of
the episode. At serving time \method{} samples one
action chunk per query, while EXPO-FT serves the best of six sampled candidates per query.

\paragraph{Differences from the EXPO-FT recipe.} We run EXPO-FT with three changes to its released
recipe, so that both methods follow the same protocol. First, neither method uses the corrective
actions from a human operator that the EXPO-FT recipe includes. Second, its base policy is kept frozen and is
not updated with its original training objective. Third, the critic of both methods reads the same
mean-pooled VLM feature in place of the separate ResNet encoder of EXPO-FT, to reduce the compute
cost. The remaining settings are shared with \method{} (\cref{tab:rby1:sqam}), and EXPO-FT uses
best-of-$N$ sampling with $N=6$ at training and serving time.

\subsection{Training details}
\label{app:rby1:training}

Since \method{} updates the pretrained policy itself, its actor learning rate is small. EXPO-FT trains a residual edit policy and a Q head on top of the frozen base policy, with the
changes listed in \cref{app:rby1:eval}, and its online run continues from its own 15K offline
checkpoint. Its learning rates are those of EXPO-FT, and its edit scale is the smallest value that work uses. \cref{tab:rby1:sqam} lists all hyperparameters.

\begin{table}[!htbp]
\centering
\small
\caption{\textbf{Hyperparameters for the real-robot runs.} The shared block applies to both methods,
and the blocks below give each method's own values.}
\label{tab:rby1:sqam}
\setlength{\tabcolsep}{5pt}
\renewcommand{\arraystretch}{0.95}
\begin{tabular}{@{}ll@{}}
\toprule
\textbf{Parameter} & \textbf{Value} \\
\midrule
\multicolumn{2}{@{}l}{Shared} \\
Offline RL steps                          & 15{,}000 \\
Online RL steps                           & 10{,}000 \\
Batch size                                & 128 offline, 192 online \\
Policy delay                              & 5 \\
UTD ratio                                 & 1 \\
Action horizon                            & 16 \\
Action DoF                                & 54 \\
Critic ensemble size ($K$)                & 10 \\
Network width                             & 512 \\
Network depth                             & 4 hidden layers \\
Critic input                              & mean-pooled VLM feature, proprioception, action chunk \\
\midrule
\multicolumn{2}{@{}l}{\method{}} \\
Actor / critic learning rate              & $10^{-5}$ / $10^{-4}$ \\
Value penalty coefficient ($c$)           & 0.01 \\
KL budget ($\varepsilon_{\mathrm{KL}}$)   & 0.5 \\
\midrule
\multicolumn{2}{@{}l}{EXPO-FT} \\
Actor / critic learning rate              & $3 \times 10^{-4}$ / $3 \times 10^{-4}$ \\
Best-of-$N$ sampling (training and serving) & 6 \\
Edit scale                            & 0.05 \\
\bottomrule
\end{tabular}
\end{table}

\subsection{Additional results}
\label{app:rby1:results}

\cref{fig:vlacurves} tracks the success rate of the training rollouts during the online phase,
and the rollout success rate of \method{} tends to increase as training proceeds on all three tasks. \cref{fig:flip_failure_mode,fig:straw_failure_mode,fig:pot_failure_mode} show the final frames of the evaluation
episodes and the main failure mode of each method. The difference is clearest on \texttt{flip-plastic-bag}. The superimposed final frames of SFT show
many episodes that end in the middle of flipping the bag, and these become visibly fewer under
\method{}, consistent with its higher success rate. EXPO-FT instead often drops the bag while flipping
it. \cref{fig:success_len} compares the duration
of the successful episodes. \method{} completes \texttt{flip-plastic-bag} and \texttt{put-and-close}
faster than SFT and \texttt{place-straw} in a comparable time.

\begin{figure}[t]
\begin{center}
\includegraphics[width=0.9\linewidth]{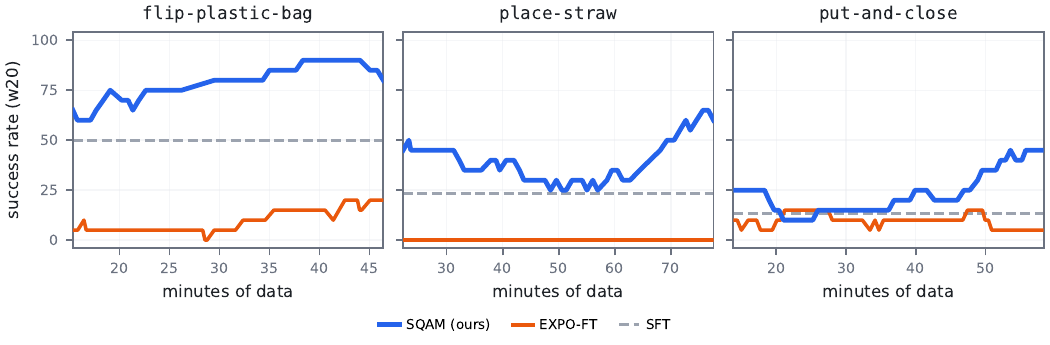}
\end{center}
\vspace{-8pt}
\caption{Success rate of the training rollouts during the online RL phase, over a trailing window of
20 episodes, against minutes of robot data collected. The dashed line is the SFT policy's held-out
success rate. \method{} works at the scale of a vision-language-action policy, improving over SFT on all
three tasks.}
\label{fig:vlacurves}
\end{figure}

\begin{figure}[ht]
\begin{center}
\includegraphics[width=0.88\linewidth]{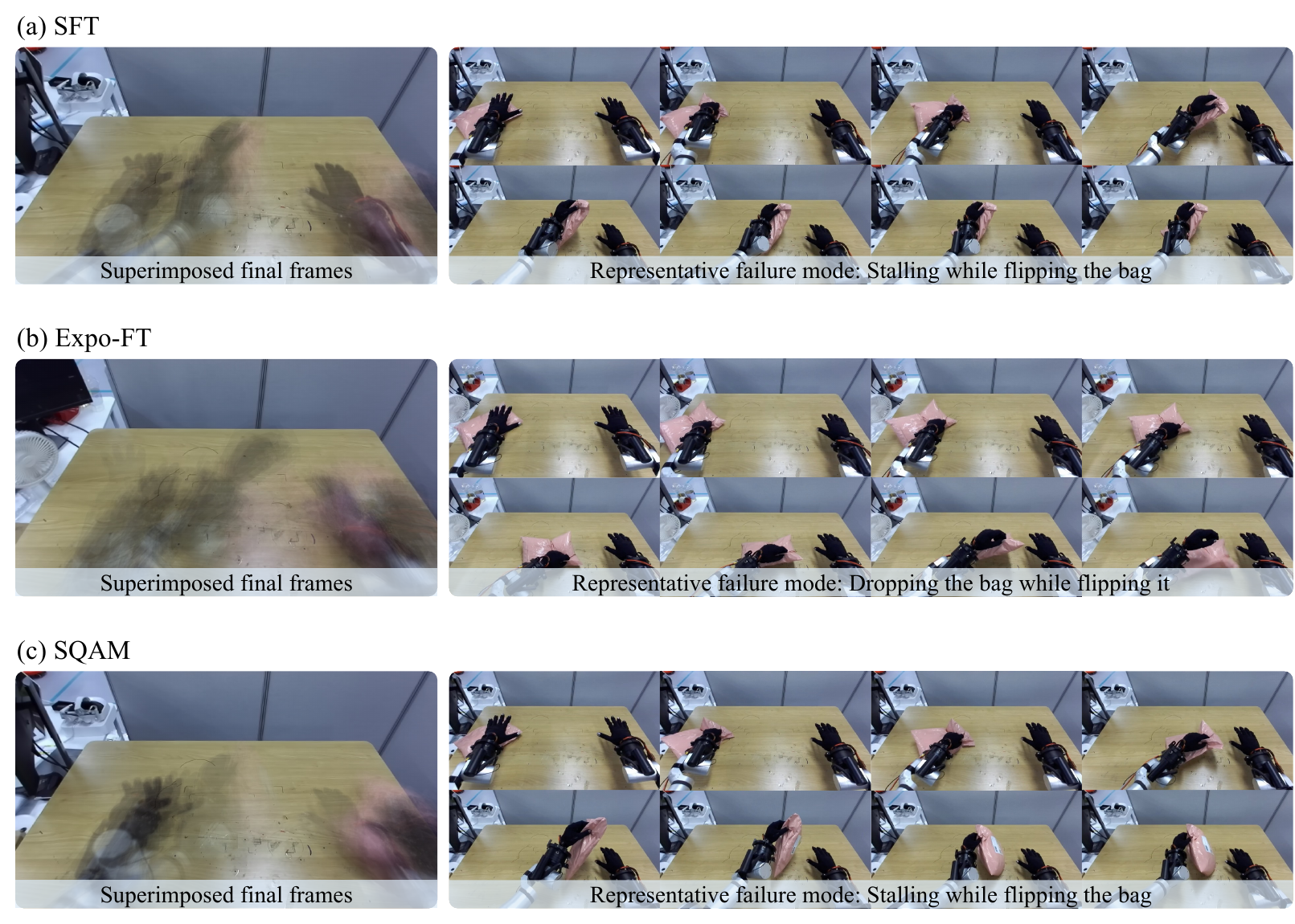}
\end{center}
\vspace{-8pt}
\caption{Evaluation episodes on \texttt{flip-plastic-bag}. Left: superimposed final frames of all episodes. Right: the main failure mode of each method.}
\label{fig:flip_failure_mode}
\end{figure}

\begin{figure}[ht]
\begin{center}
\includegraphics[width=0.88\linewidth]{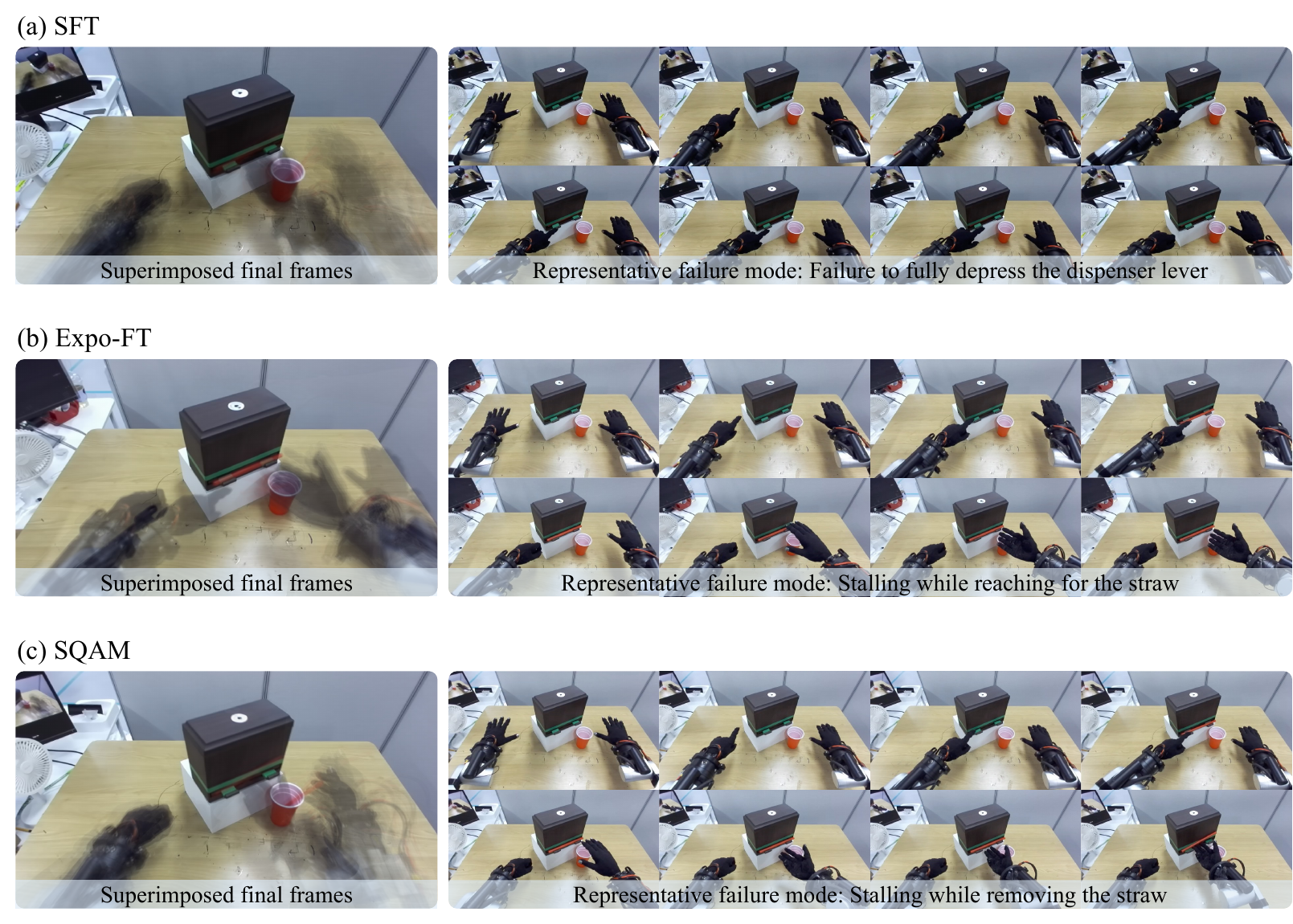}
\end{center}
\vspace{-8pt}
\caption{Evaluation episodes on \texttt{place-straw}. Left: superimposed final frames of all episodes. Right: the main failure mode of each method.}
\label{fig:straw_failure_mode}
\end{figure}

\begin{figure}[ht]
\begin{center}
\includegraphics[width=0.88\linewidth]{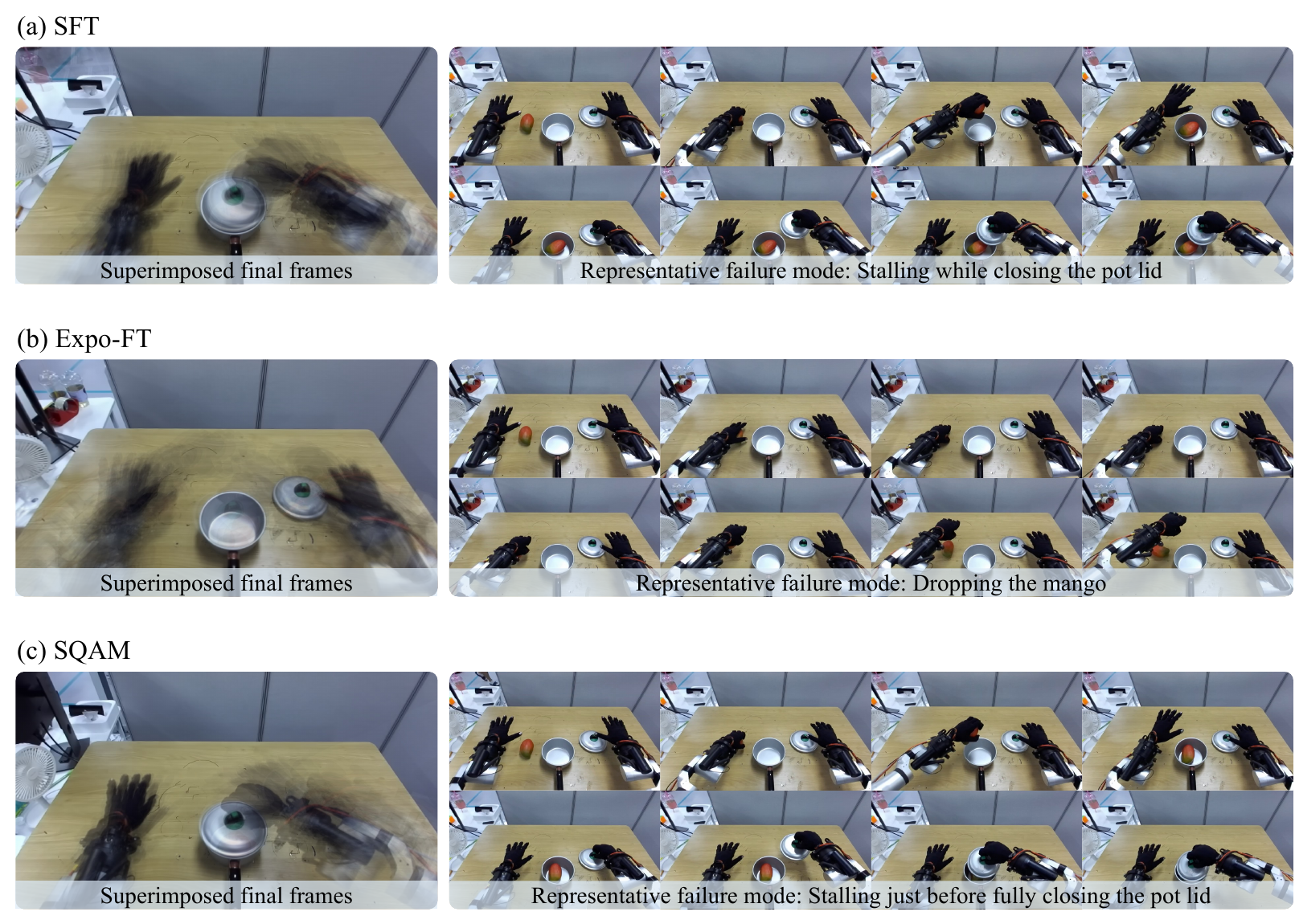}
\end{center}
\vspace{-8pt}
\caption{Evaluation episodes on \texttt{put-and-close}. Left: superimposed final frames of all episodes. Right: the main failure mode of each method.}
\label{fig:pot_failure_mode}
\end{figure}

\begin{figure}[ht]
\begin{center}
\includegraphics[width=\linewidth]{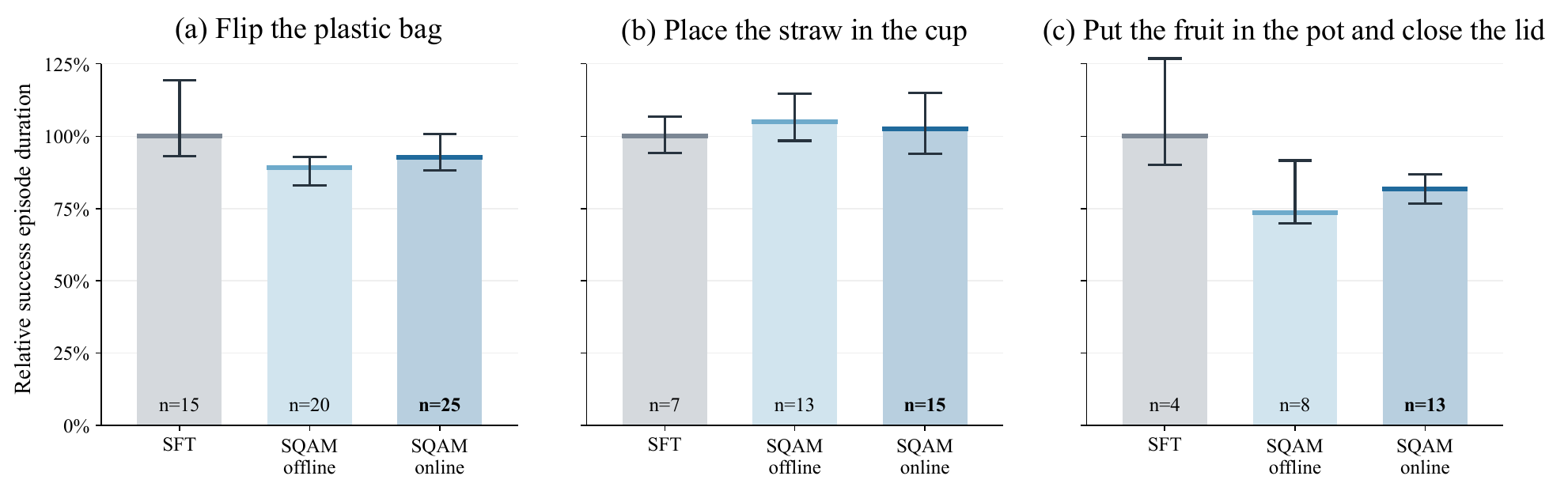}
\end{center}
\vspace{-8pt}
\caption{Success episode durations normalized by each task’s SFT median (100\%). Bars show medians, error bars indicate the 25th–75th percentiles, and \(n\) denotes the number of success episodes.}
\label{fig:success_len}
\end{figure}

\clearpage

\section{Additional experiments}
\label{app:additional}

\subsection{The cost of one update}
\label{app:cost}

To measure the compute that the scalar adjoint saves, \cref{fig:cost} compares the wall-clock
time of one \method{} update with that of the same update using the exact adjoint. The exact adjoint takes a vector--Jacobian
product through the policy at each of the $F$ flow steps, and the scalar adjoint takes none. The left
panel widens the actor with the critic fixed. The exact adjoint remains about half again as slow at
every width, so the saving persists on larger policies. The right panel varies the number of flow steps
$F$. The gap grows with $F$, since the cost of the exact adjoint grows with the number of
vector--Jacobian products while the rest of the update does not.

\begin{figure}[ht]
\begin{center}
\includegraphics[width=\linewidth]{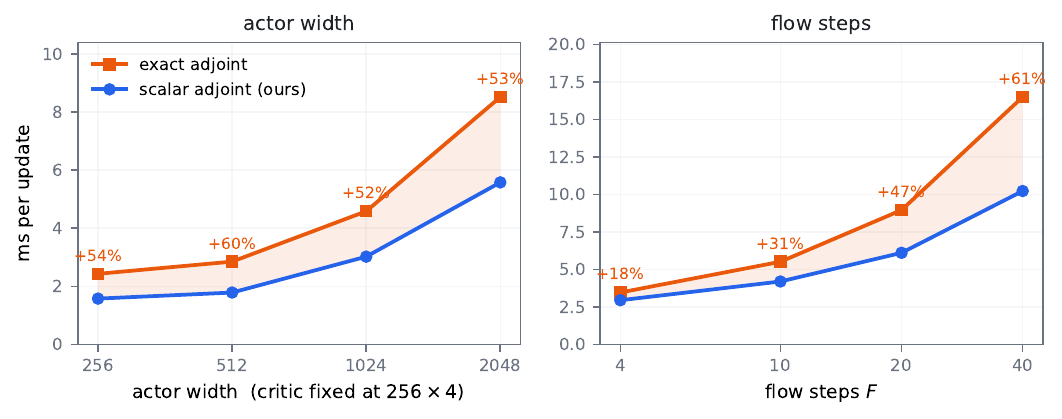}
\end{center}
\caption{Wall-clock time of one update with the exact and the scalar adjoint. \textbf{Left:} actor
width, with the critic fixed at $256\times4$. \textbf{Right:} number of flow steps $F$. Percentages give
the extra time of the exact adjoint.}
\label{fig:cost}
\end{figure}

\subsection{The value penalty inside TRQAM}
\label{app:chainpen}

\cref{sec:gain} explains the effect of the value penalty by where the scalar adjoint uses the
critic, which is its gradient at the policy-generated action, used directly at every flow step. The
exact adjoint instead passes this gradient through the velocity Jacobian at every flow step, so if the
explanation holds, the value penalty should help the exact adjoint less. To test this, we switch the
penalty on and off in both TRQAM and \method{} on \texttt{cube-triple} and \texttt{cube-quadruple}, the
two domains where the scalar adjoint alone underperforms. All four configurations use the same four seeds with the trust region active, and they differ only in the critic loss.

\cref{fig:chainpen} shows that the value penalty improves TRQAM as well, but far less than it
improves \method{}. At the offline endpoint, it raises TRQAM by $18$ points on \texttt{cube-triple} and
$23$ on \texttt{cube-quadruple}, and raises \method{} by $48.5$ and $78$ points. This matches the
prediction above.

\begin{figure}[t]
\begin{center}
\includegraphics[width=\linewidth]{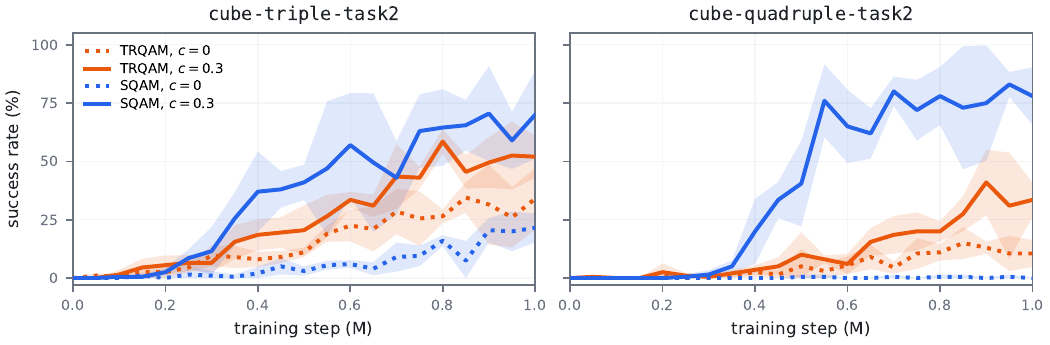}
\end{center}
\caption{TRQAM and \method{} with and without the value penalty ($c=0.3$ and $c=0$) during offline
training on task 2 of \texttt{cube-triple} and \texttt{cube-quadruple}. Four seeds per curve, with the
trust region at the KL budget of each method.}
\label{fig:chainpen}
\end{figure}

\subsection{Without the trust region}
\label{app:notrust}

\cref{sec:decomp} notes that \method{} can be used with or without the path-space trust region
of TRQAM. To run it without, we freeze the dual variable, $\eta_\lambda=0$, so that $\lambda$ stays at
a constant value as in QAM. \cref{fig:notrust} runs five constants
$\lambda\in\{0.1, 0.3, 1, 3, 10\}$ on one representative task from each of the ten domains, with the other settings, including the value
penalty coefficients, as in the main experiments and eight seeds per curve. The dashed line is \method{}
with the trust region.

Without the trust region, the effect of $\lambda$ is less predictable across domains. The best constant
differs by two orders of magnitude, for example $\lambda=0.1$ is the best value on
\texttt{antmaze-giant} but fails on \texttt{puzzle-3x3}, while $\lambda=10$ is the best value on
\texttt{puzzle-3x3} but fails on \texttt{antmaze-large}. With the trust region, the KL budget affects
performance more predictably (\cref{sec:sensitivity}) and keeps training stable, so we use it by
default.

\begin{figure}[p]
\begin{center}
\includegraphics[width=\linewidth]{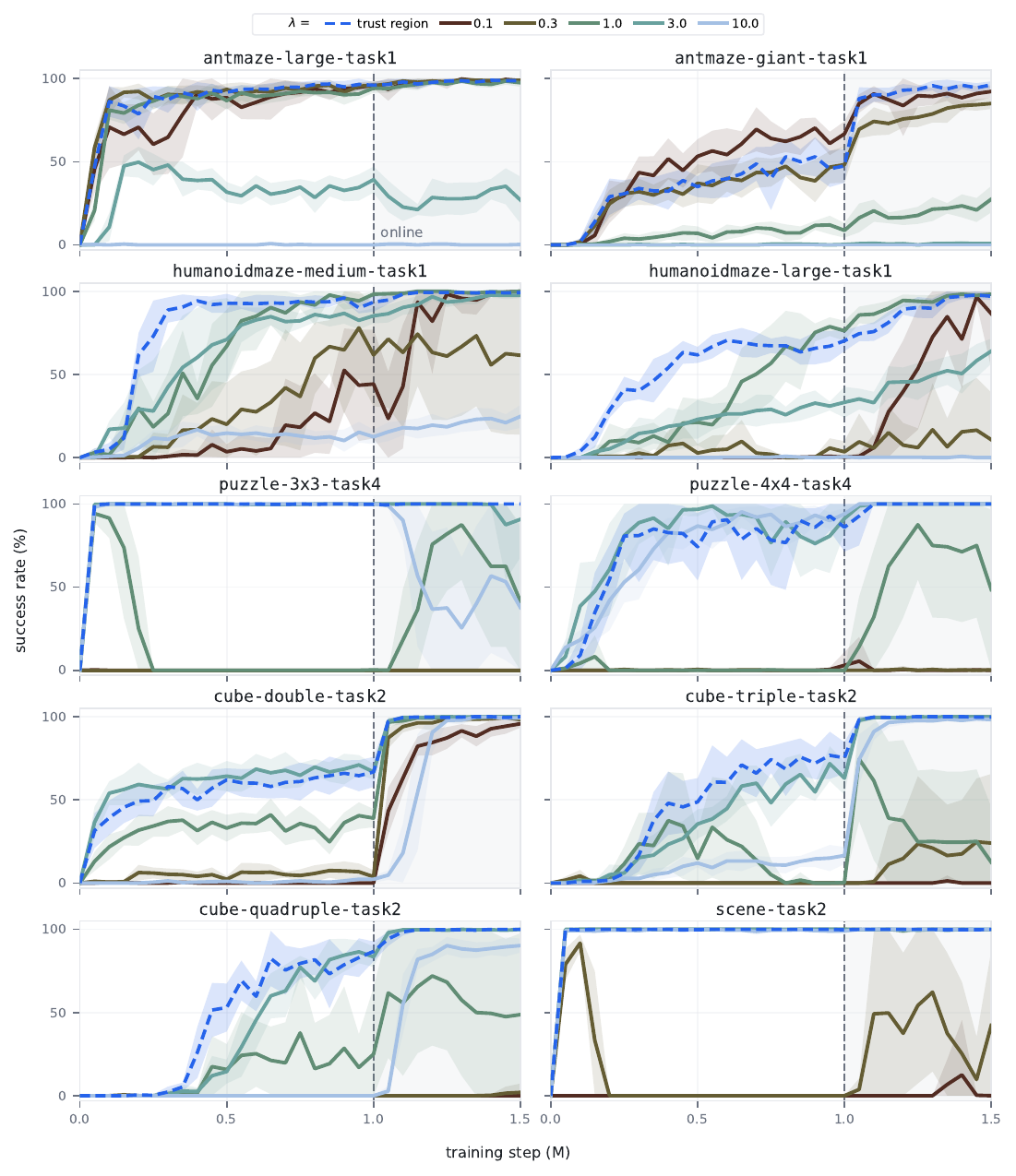}
\end{center}
\caption{\method{} with a constant $\lambda$, for five values, against \method{} with the trust region
(dashed) on one representative task from each of the ten domains. Eight seeds per curve, and bands are one standard deviation.}
\label{fig:notrust}
\end{figure}

\subsection{Batch-averaged velocity Jacobian on all domains}
\label{app:geometry}

\cref{fig:jacobian-all,fig:jacobian-all2} extend \cref{fig:jacobian} to all ten
domains at four flow times. As in \cref{fig:jacobian}, the velocity Jacobian of the pretrained
flow is averaged over a batch of states sampled from the dataset and normalized by its mean diagonal.
On every domain, the batch-averaged Jacobian is dominated by its diagonal.

\begin{figure}[p]
\begin{center}
\includegraphics[width=\linewidth]{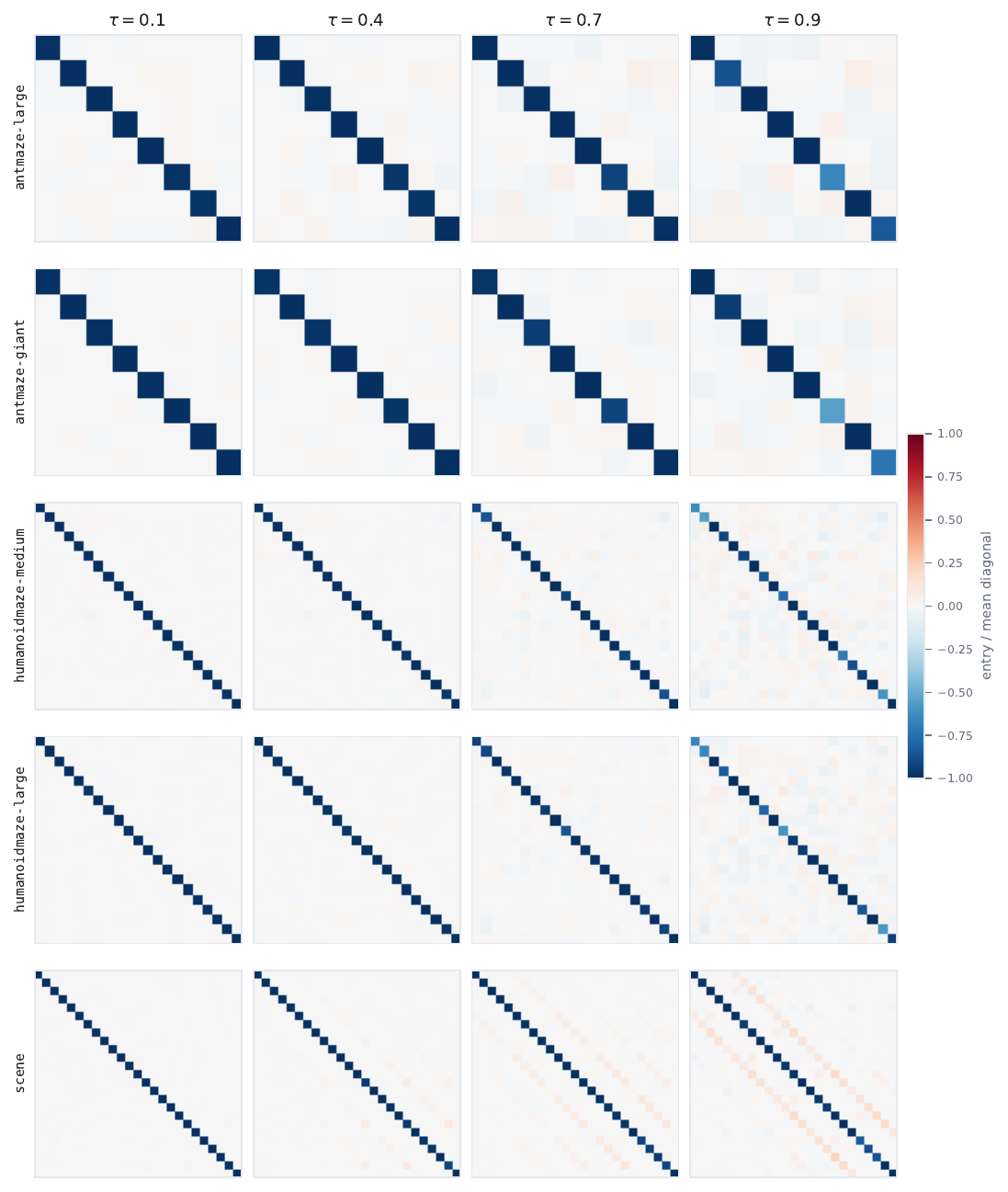}
\end{center}
\caption{The batch-averaged velocity Jacobian on the locomotion domains and \texttt{scene} at four flow
times, each panel normalized by its mean diagonal.}
\label{fig:jacobian-all}
\end{figure}

\begin{figure}[p]
\begin{center}
\includegraphics[width=\linewidth]{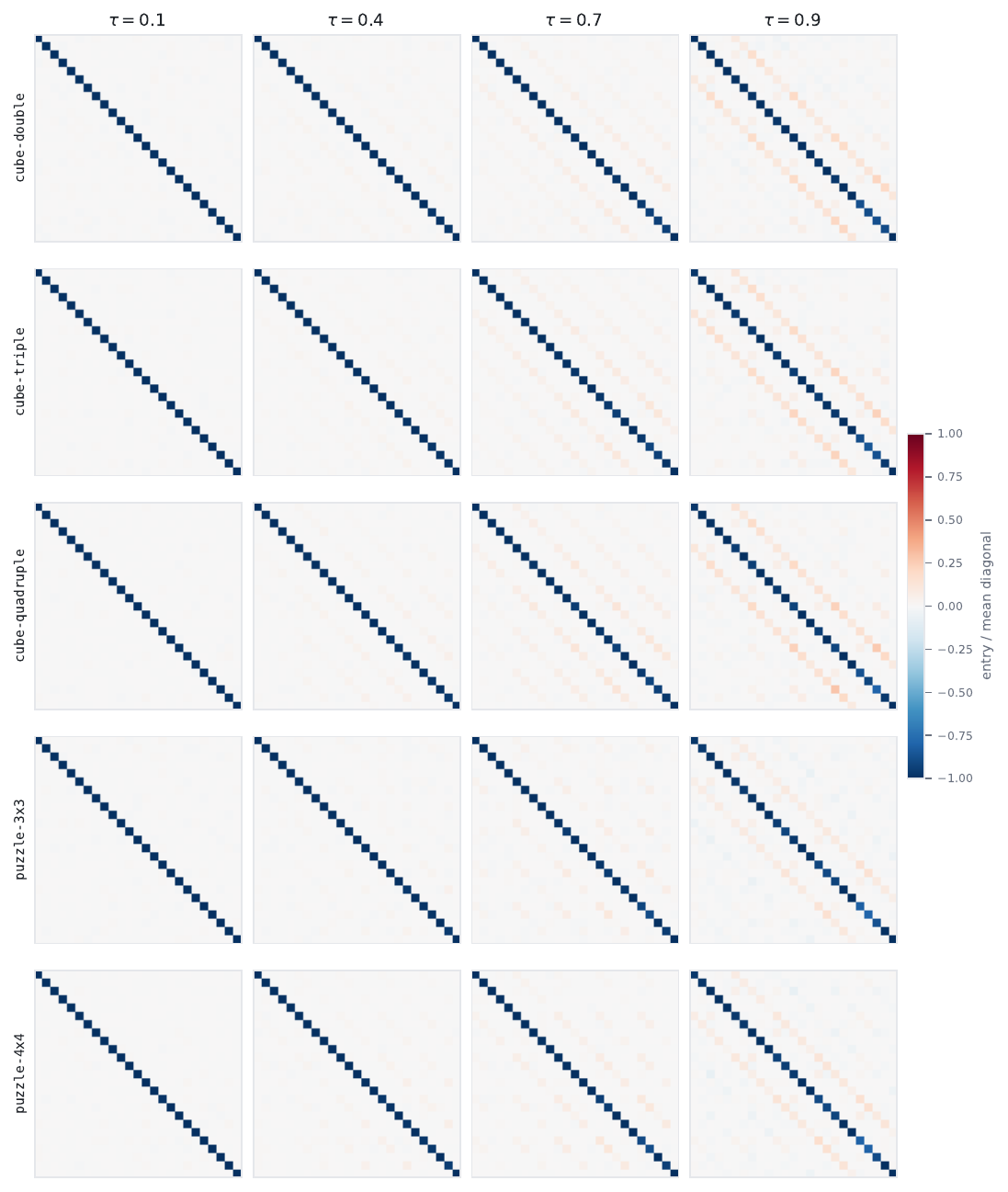}
\end{center}
\caption{The batch-averaged velocity Jacobian on the \texttt{cube} and \texttt{puzzle} domains at four
flow times, with the same normalization as \cref{fig:jacobian-all}.}
\label{fig:jacobian-all2}
\end{figure}

\subsection{Ablation studies on all domains}
\label{app:sensfull}

\cref{fig:kl-ten,fig:sweep-c} extend the ablations of \cref{sec:sensitivity} to
all ten domains, with one representative task per domain and eight seeds per setting, offline to 1M steps
and online to 1.5M. They sweep the KL budget and the value penalty coefficient $c$. The trends of the
main text hold across the ten domains. A small KL budget performs best or equally well on most domains.
The exception is \texttt{puzzle-4x4}, which prefers a larger budget, consistent with its larger state
space and with the budget that \citet{dong2026trustregionqadjoint} report for it. As in \cref{sec:sensitivity}, the value
penalty matters mainly on the domains where the scalar adjoint alone underperforms, with large gains on
\texttt{cube-triple} and \texttt{cube-quadruple} and little gain on most other domains. On \texttt{cube-double}, it lowers the offline success rate at $c=0.3$, so the main
experiments run this domain without the penalty.

\begin{figure}[p]
\begin{center}
\includegraphics[width=\linewidth]{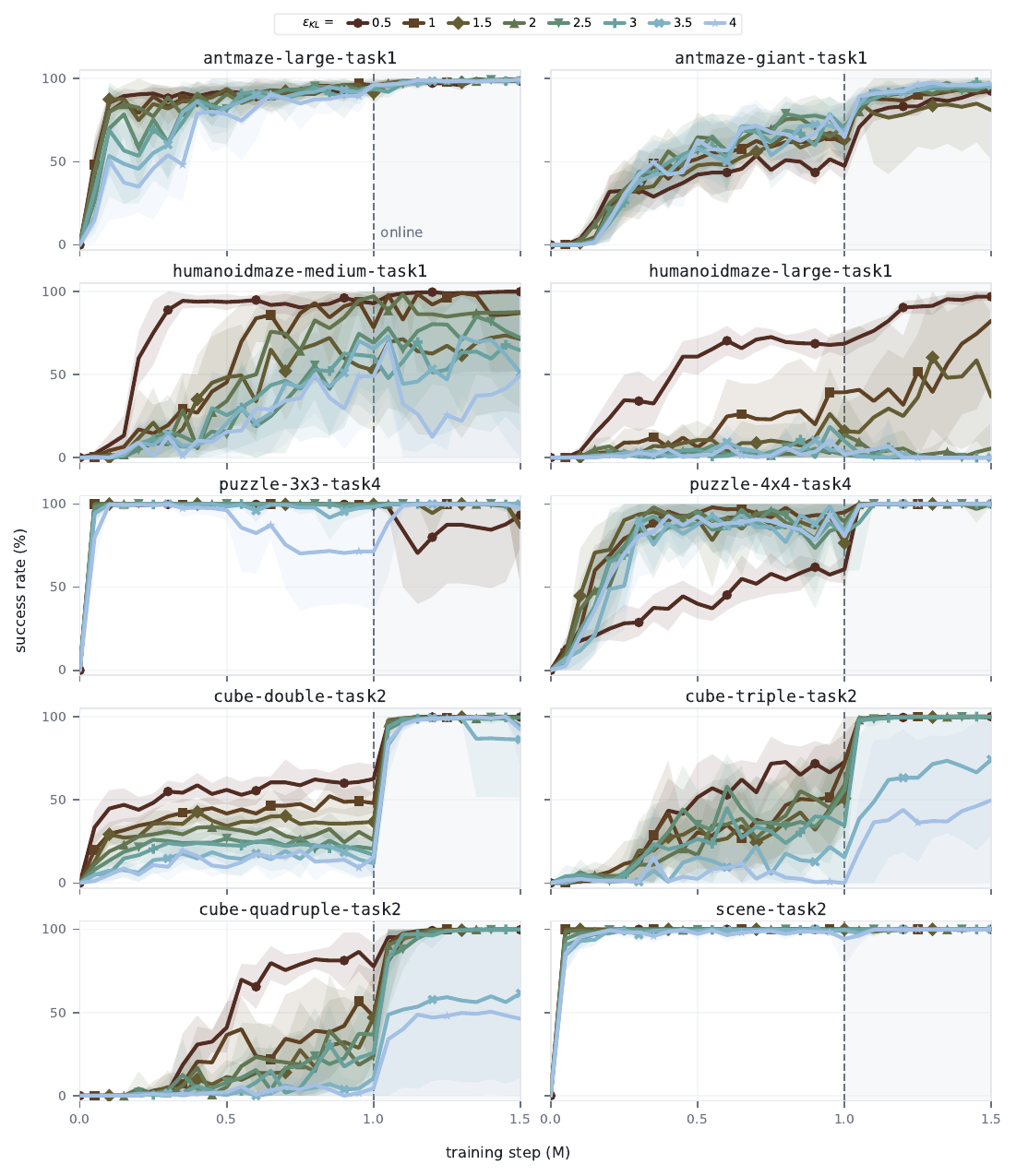}
\end{center}
\caption{The KL budget $\varepsilon_{\mathrm{KL}}$ swept from $0.5$ to $4$ on all ten domains. Eight
seeds per budget, and bands are one standard deviation.}
\label{fig:kl-ten}
\end{figure}

\begin{figure}[p]
\begin{center}
\includegraphics[width=\linewidth]{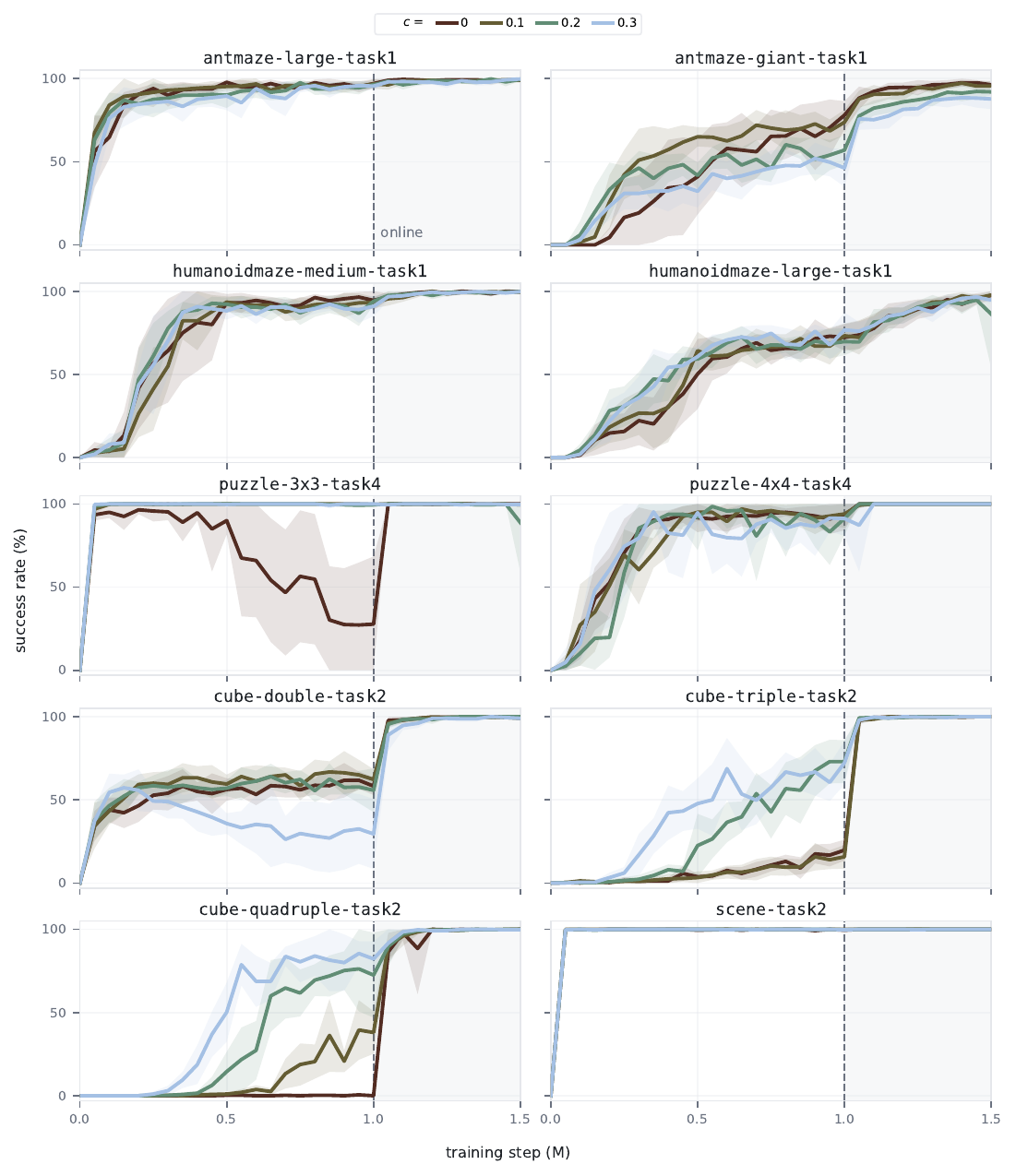}
\end{center}
\caption{The value penalty coefficient $c$ swept over $\{0, 0.1, 0.2, 0.3\}$ on all ten domains. Eight
seeds per coefficient, and bands are one standard deviation.}
\label{fig:sweep-c}
\end{figure}

\subsection{Value learning schemes on all domains}
\label{app:schemesfull}

\cref{fig:sweep-gn,fig:sweep-rb,fig:sweep-iq} extend the comparison of value
learning schemes in \cref{sec:gain} to all ten domains, with the same tasks, seeds, and steps as
\cref{app:sensfull}. They sweep the gradient-norm penalty, ReBRAC~\citep{tarasov2023revisitingminimalistapproachoffline}, and the in-sample critic~\citep{kostrikov2021offlinereinforcementlearningimplicit} over
their own coefficients, and the unregularized critic (\textit{none}) is drawn as a dashed reference. Consistent with \cref{sec:gain}, the
gradient-norm penalty also helps, most visibly on \texttt{puzzle-3x3} and \texttt{cube-quadruple},
but degrades several domains at its largest coefficient, ReBRAC stays close to the unregularized critic except at its largest $\beta$, and the
in-sample critic falls below the unregularized critic on most domains. The value penalty sweep of
\cref{fig:sweep-c} completes the comparison.

\begin{figure}[p]
\begin{center}
\includegraphics[width=\linewidth]{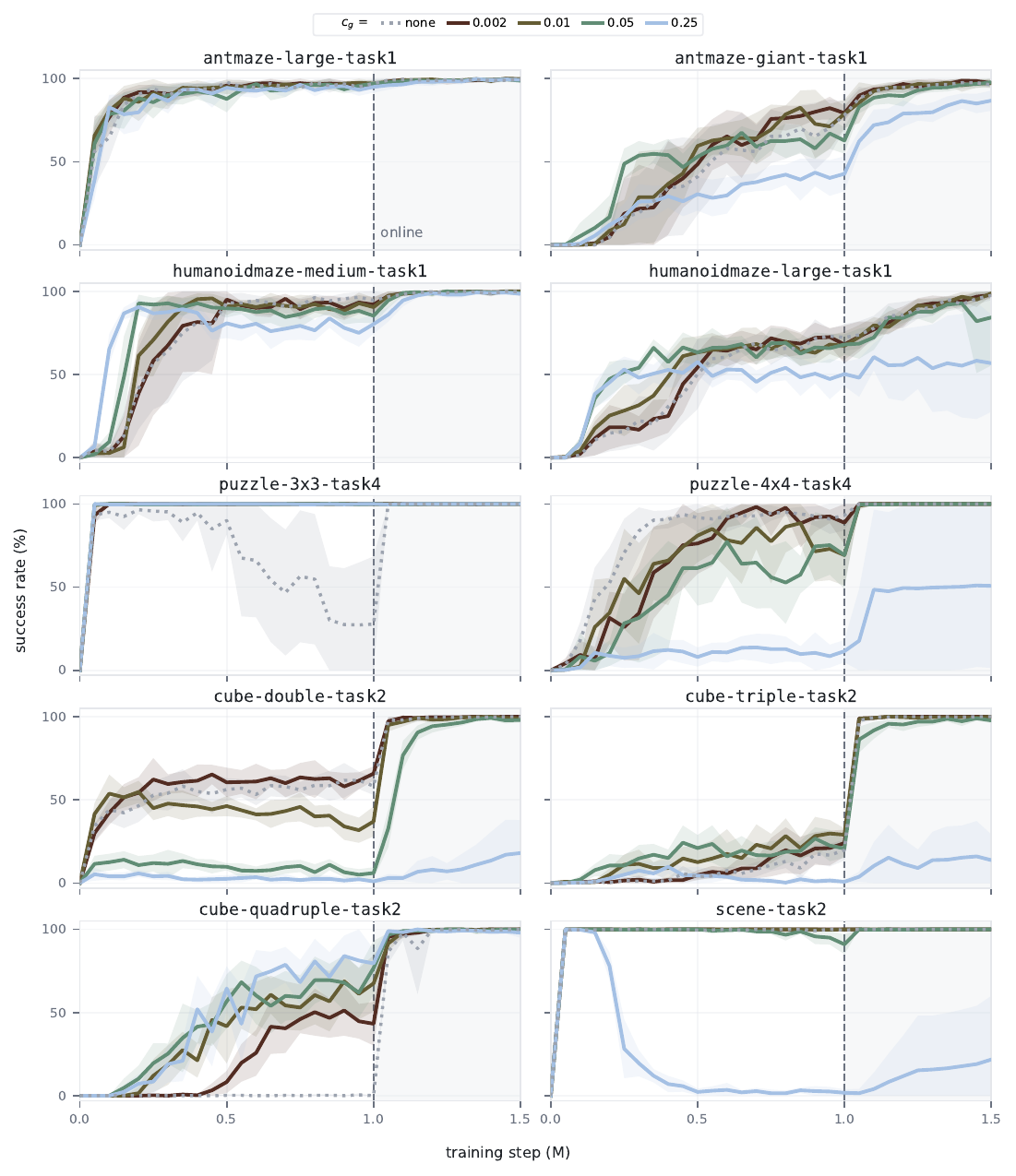}
\end{center}
\caption{The gradient-norm penalty coefficient $c_g$ swept over $\{0.002, 0.01, 0.05, 0.25\}$ on all ten domains,
with the unregularized critic (\textit{none}) as the dashed reference. Eight seeds per coefficient, and bands are one
standard deviation.}
\label{fig:sweep-gn}
\end{figure}

\begin{figure}[p]
\begin{center}
\includegraphics[width=\linewidth]{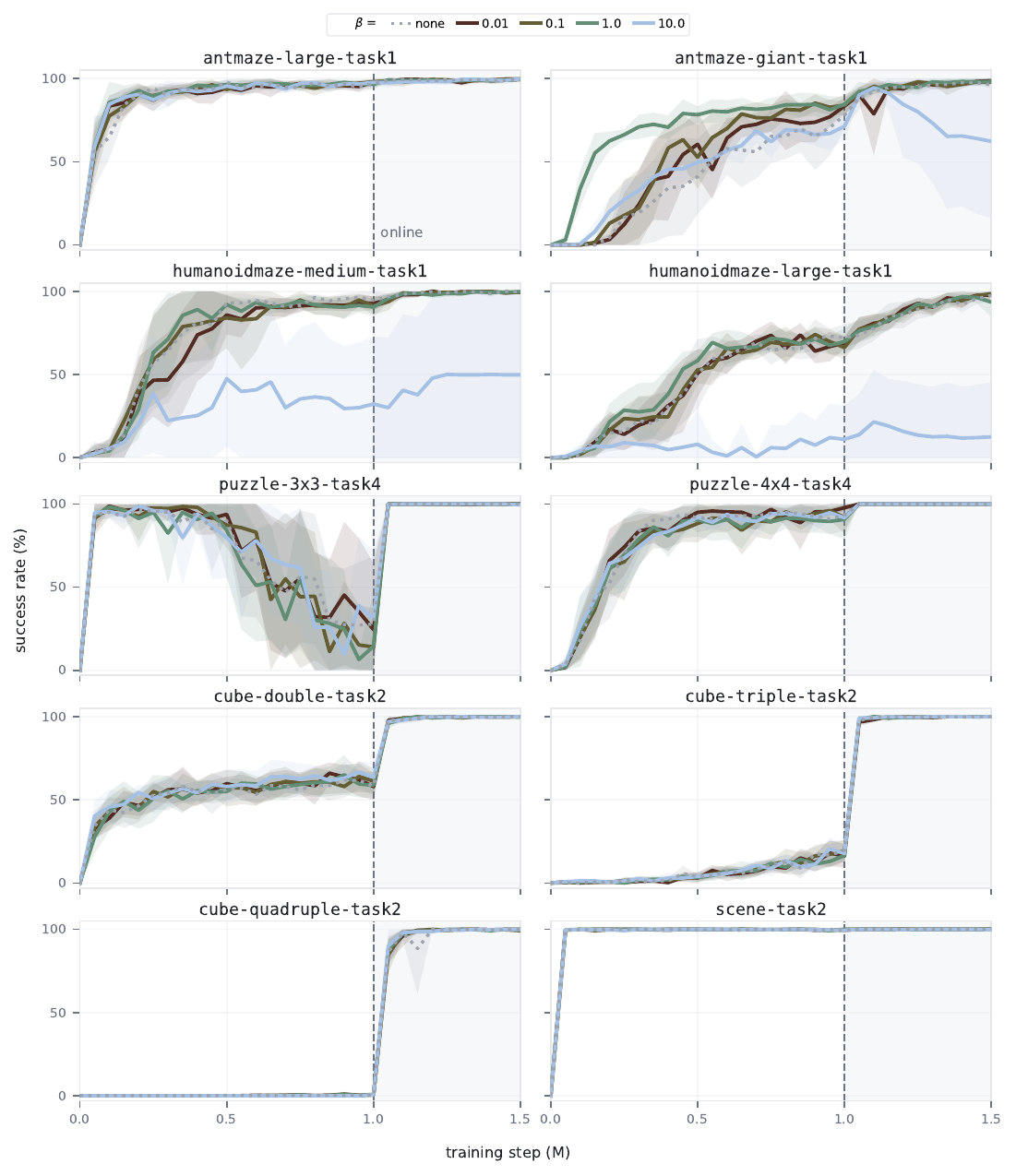}
\end{center}
\caption{ReBRAC's bootstrap penalty swept over $\beta \in \{0.01, 0.1, 1, 10\}$ on all ten domains,
with the unregularized critic (\textit{none}) as the dashed reference. Eight seeds per coefficient, and bands are one
standard deviation.}
\label{fig:sweep-rb}
\end{figure}

\begin{figure}[p]
\begin{center}
\includegraphics[width=\linewidth]{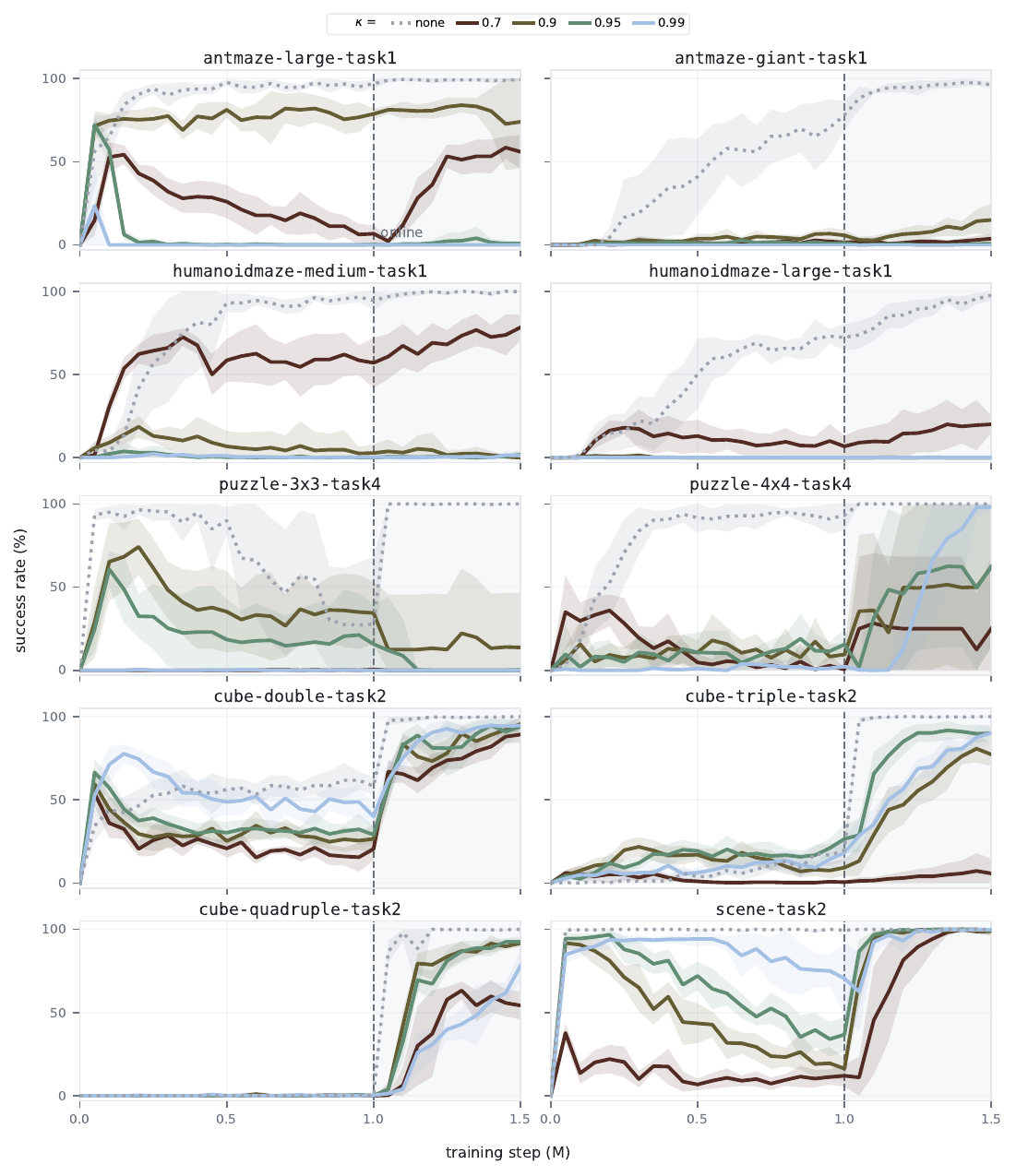}
\end{center}
\caption{The in-sample critic swept over the expectile $\kappa \in \{0.7, 0.9, 0.95, 0.99\}$ on all ten
domains, with the unregularized critic (\textit{none}) as the dashed reference. Eight seeds per expectile, and bands are
one standard deviation.}
\label{fig:sweep-iq}
\end{figure}

\section{Full OGBench results}
\label{app:benchmark}

\cref{tab:full-results} reports the offline success rate of every method on all fifty tasks at 1M
steps, the per-task breakdown behind \cref{tab:offline}. \cref{fig:bench-1,fig:bench-2} show the corresponding success rate curves, 1M offline steps followed by 0.5M
online, with eight seeds per method. \method{} uses the hyperparameters of
\cref{app:hyperparams} on every task.

\begin{table}[p]
    \centering
    \caption{\textbf{Full offline results at 1M training steps (8 seeds).} Mean success rate (\%) $\pm$
    one standard deviation.}
    \label{tab:full-results}
    \setlength{\abovecaptionskip}{4pt}

    \newcommand{\gval}[2]{{\color{gray!70}#1}\,{\color{gray!50}\scriptsize$\pm$#2}}
    \newcommand{\bmax}[2]{\underline{#1}\,{\color{gray!50}\scriptsize$\pm$#2}}
    
    \newcommand{\gavg}[1]{{\color{gray!70}#1}}
    \newcommand{\bavg}[1]{\underline{#1}}
    \newcommand{\ourmax}[2]{{\color{GoogleBlue}\textbf{\underline{#1}}}\,{\color{gray!50}\scriptsize$\pm$#2}}
    \newcommand{\ouravg}[1]{{\color{GoogleBlue}\textbf{\underline{#1}}}}

    \makebox[\textwidth][c]{
        \begin{minipage}{\textwidth}
            \centering
            \renewcommand{\arraystretch}{1.0}
            \setlength{\aboverulesep}{2pt}
            \setlength{\belowrulesep}{2pt}

            \adjustbox{max width=\textwidth, max totalheight=0.97\textheight}{
                \begin{tabular}{cl ccccccc >{\columncolor{ourlightblue}}c}
                \toprule
                 &  & \texttt{FQL} & \texttt{CGQL-L} & \texttt{DSRL} & \texttt{IFQL} & \texttt{QAM} & \texttt{QAM-E} & \texttt{TRQAM} & \texttt{SQAM} \\
                \midrule
                 & \texttt{task1}   & \gval{24}{44} & \gval{62}{38} & \gval{45}{12} & \gval{7}{11} & \bmax{96}{3} & \gval{93}{5} & \gval{95}{4} & \ourmax{96}{2} \\
                 & \texttt{task2}   & \gval{0}{0} & \gval{21}{36} & \gval{73}{7} & \gval{9}{14} & \gval{36}{49} & \bmax{91}{4} & \gval{87}{6} & \gval{85}{4} \\
                 & \texttt{task3}   & \gval{73}{14} & \gval{88}{5} & \gval{84}{5} & \gval{62}{24} & \gval{86}{7} & \bmax{93}{3} & \bmax{93}{3} & \ourmax{93}{4} \\
                 & \texttt{task4}   & \gval{0}{0} & \gval{0}{0} & \gval{0}{0} & \gval{35}{22} & \gval{0}{0} & \gval{66}{9} & \gval{75}{17} & \ourmax{88}{7} \\
                 & \texttt{task5}   & \gval{92}{4} & \gval{68}{42} & \gval{63}{8} & \gval{34}{24} & \gval{94}{3} & \gval{89}{5} & \bmax{96}{3} & \gval{93}{3} \\
                \multirow{-6}{*}{\texttt{antmaze-large}} & \texttt{agg. (5 tasks)}   & \gval{38}{9} & \gval{48}{7} & \gval{53}{2} & \gval{29}{8} & \gval{62}{9} & \gval{86}{3} & \gval{89}{4} & \ourmax{91}{3} \\
                \midrule
                 & \texttt{task1}   & \gval{1}{1} & \gval{11}{12} & \gval{2}{3} & \gval{0}{0} & \gval{37}{9} & \gval{0}{0} & \bmax{48}{5} & \ourmax{48}{11} \\
                 & \texttt{task2}   & \gval{0}{0} & \gval{14}{12} & \gval{0}{0} & \gval{43}{13} & \gval{18}{19} & \gval{0}{0} & \gval{21}{12} & \ourmax{86}{3} \\
                 & \texttt{task3}   & \gval{0}{0} & \gval{0}{0} & \gval{0}{0} & \gval{14}{8} & \gval{2}{2} & \gval{0}{0} & \gval{5}{5} & \ourmax{30}{10} \\
                 & \texttt{task4}   & \gval{0}{0} & \gval{0}{0} & \gval{0}{0} & \gval{1}{2} & \gval{20}{24} & \gval{0}{0} & \gval{58}{13} & \ourmax{61}{8} \\
                 & \texttt{task5}   & \gval{10}{29} & \gval{8}{23} & \gval{2}{2} & \gval{2}{2} & \gval{69}{8} & \gval{28}{38} & \gval{75}{7} & \ourmax{83}{8} \\
                \multirow{-6}{*}{\texttt{antmaze-giant-10M}} & \texttt{agg. (5 tasks)}   & \gval{2}{6} & \gval{7}{5} & \gval{1}{1} & \gval{12}{3} & \gval{29}{4} & \gval{6}{8} & \gval{41}{4} & \ourmax{62}{4} \\
                \midrule
                
                 & \texttt{task1}   & \gval{84}{21} & \gval{84}{9} & \gval{24}{23} & \gval{93}{3} & \gval{33}{34} & \gval{28}{30} & \gval{87}{5} & \ourmax{94}{6} \\
                 & \texttt{task2}   & \gval{99}{1} & \gval{99}{1} & \gval{88}{6} & \gval{94}{4} & \bmax{100}{1} & \gval{98}{2} & \gval{89}{7} & \gval{96}{3} \\
                 & \texttt{task3}   & \gval{89}{9} & \gval{0}{0} & \gval{64}{27} & \gval{96}{4} & \gval{90}{8} & \gval{74}{18} & \gval{90}{6} & \ourmax{98}{2} \\
                 & \texttt{task4}   & \gval{0}{0} & \gval{0}{0} & \gval{0}{0} & \bmax{82}{8} & \gval{0}{0} & \gval{0}{0} & \gval{61}{7} & \ourmax{82}{7} \\
                 & \texttt{task5}   & \gval{99}{1} & \bmax{100}{1} & \gval{88}{5} & \gval{99}{2} & \bmax{100}{1} & \gval{99}{1} & \gval{95}{4} & \ourmax{100}{1} \\
                \multirow{-6}{*}{\texttt{humanoidmaze-medium}} & \texttt{agg. (5 tasks)}   & \gval{74}{5} & \gval{57}{2} & \gval{53}{10} & \gval{93}{2} & \gval{64}{7} & \gval{60}{6} & \gval{84}{3} & \ourmax{94}{2} \\
                \midrule
                
                 & \texttt{task1}   & \gval{1}{2} & \gval{16}{11} & \gval{1}{1} & \gval{32}{5} & \gval{6}{10} & \gval{9}{14} & \gval{45}{8} & \ourmax{70}{7} \\
                 & \texttt{task2}   & \gval{0}{0} & \gval{0}{0} & \gval{0}{0} & \gval{6}{9} & \gval{0}{0} & \gval{0}{0} & \gval{10}{6} & \ourmax{25}{13} \\
                 & \texttt{task3}   & \gval{7}{4} & \gval{16}{8} & \gval{2}{3} & \bmax{76}{6} & \gval{7}{5} & \gval{3}{4} & \gval{34}{15} & \gval{57}{17} \\
                 & \texttt{task4}   & \gval{0}{0} & \gval{0}{1} & \gval{0}{0} & \gval{35}{26} & \gval{0}{0} & \gval{0}{0} & \gval{44}{13} & \ourmax{60}{14} \\
                 & \texttt{task5}   & \gval{0}{1} & \gval{0}{0} & \gval{1}{1} & \gval{0}{0} & \gval{5}{11} & \gval{9}{18} & \gval{45}{6} & \ourmax{59}{14} \\
                \multirow{-6}{*}{\texttt{humanoidmaze-large}} & \texttt{agg. (5 tasks)}   & \gval{2}{1} & \gval{6}{3} & \gval{1}{1} & \gval{30}{7} & \gval{4}{3} & \gval{4}{5} & \gval{36}{4} & \ourmax{54}{5} \\
                \midrule
                
                 & \texttt{task1}   & \bmax{100}{0} & \bmax{100}{0} & \bmax{100}{0} & \gval{99}{1} & \bmax{100}{0} & \bmax{100}{0} & \bmax{100}{0} & \ourmax{100}{0} \\
                 & \texttt{task2}   & \bmax{100}{1} & \gval{99}{1} & \bmax{100}{0} & \gval{2}{2} & \bmax{100}{0} & \bmax{100}{0} & \bmax{100}{0} & \ourmax{100}{1} \\
                 & \texttt{task3}   & \gval{82}{5} & \gval{93}{7} & \bmax{100}{1} & \gval{77}{7} & \gval{94}{4} & \gval{93}{4} & \bmax{100}{1} & \ourmax{99}{2} \\
                 & \texttt{task4}   & \gval{60}{22} & \gval{0}{0} & \bmax{100}{0} & \gval{2}{2} & \gval{26}{21} & \gval{22}{30} & \gval{93}{6} & \ourmax{99}{2} \\
                 & \texttt{task5}   & \bmax{8}{7} & \gval{0}{0} & \gval{0}{0} & \gval{0}{0} & \gval{0}{0} & \gval{0}{0} & \gval{0}{1} & \gval{0}{0} \\
                \multirow{-6}{*}{\texttt{scene}} & \texttt{agg. (5 tasks)}   & \gval{70}{5} & \gval{58}{1} & \bmax{80}{0} & \gval{36}{1} & \gval{64}{4} & \gval{63}{6} & \gval{79}{1} & \ourmax{79}{0} \\
                \midrule
                
                 & \texttt{task1}   & \gval{87}{11} & \gval{0}{0} & \bmax{100}{0} & \bmax{100}{1} & \gval{75}{15} & \gval{99}{2} & \bmax{100}{0} & \ourmax{100}{0} \\
                 & \texttt{task2}   & \gval{37}{39} & \gval{0}{0} & \bmax{100}{0} & \gval{23}{7} & \gval{0}{0} & \bmax{100}{0} & \bmax{100}{0} & \ourmax{100}{0} \\
                 & \texttt{task3}   & \gval{0}{0} & \gval{0}{0} & \bmax{100}{0} & \gval{82}{11} & \gval{0}{0} & \gval{76}{11} & \bmax{100}{0} & \ourmax{100}{1} \\
                 & \texttt{task4}   & \gval{0}{0} & \gval{0}{0} & \bmax{100}{0} & \gval{29}{8} & \gval{0}{0} & \gval{87}{6} & \bmax{100}{1} & \ourmax{100}{1} \\
                 & \texttt{task5}   & \gval{3}{6} & \gval{0}{0} & \bmax{100}{0} & \gval{86}{9} & \gval{0}{1} & \gval{84}{20} & \bmax{100}{0} & \ourmax{100}{0} \\
                \multirow{-6}{*}{\texttt{puzzle-3x3}} & \texttt{agg. (5 tasks)}   & \gval{25}{10} & \gval{0}{0} & \bmax{100}{0} & \gval{64}{4} & \gval{15}{3} & \gval{89}{4} & \bmax{100}{0} & \ourmax{100}{0} \\
                \midrule
                
                 & \texttt{task1}   & \gval{3}{3} & \gval{1}{1} & \gval{90}{5} & \gval{74}{11} & \gval{3}{4} & \gval{90}{10} & \bmax{100}{0} & \ourmax{97}{6} \\
                 & \texttt{task2}   & \gval{6}{9} & \gval{0}{0} & \gval{9}{9} & \gval{4}{3} & \gval{1}{1} & \gval{43}{22} & \bmax{99}{1} & \gval{73}{30} \\
                 & \texttt{task3}   & \gval{23}{15} & \gval{0}{0} & \gval{81}{20} & \gval{84}{6} & \gval{2}{3} & \gval{57}{10} & \bmax{98}{2} & \gval{81}{12} \\
                 & \texttt{task4}   & \gval{7}{5} & \gval{0}{0} & \gval{76}{28} & \gval{46}{13} & \gval{0}{1} & \gval{31}{16} & \bmax{100}{1} & \gval{86}{16} \\
                 & \texttt{task5}   & \gval{9}{16} & \gval{0}{0} & \gval{49}{43} & \gval{3}{1} & \gval{1}{1} & \gval{50}{22} & \bmax{98}{5} & \gval{90}{17} \\
                \multirow{-6}{*}{\texttt{puzzle-4x4-10M}} & \texttt{agg. (5 tasks)}   & \gval{9}{7} & \gval{0}{0} & \gval{61}{8} & \gval{42}{4} & \gval{1}{1} & \gval{54}{8} & \bmax{99}{1} & \gval{85}{6} \\
                \midrule
                
                 & \texttt{task1}   & \gval{75}{10} & \gval{83}{4} & \gval{85}{7} & \gval{18}{7} & \gval{96}{3} & \gval{94}{4} & \bmax{98}{2} & \gval{94}{5} \\
                 & \texttt{task2}   & \gval{59}{11} & \gval{66}{6} & \gval{85}{5} & \gval{9}{2} & \gval{81}{6} & \gval{82}{7} & \bmax{92}{3} & \gval{67}{10} \\
                 & \texttt{task3}   & \gval{39}{8} & \gval{72}{7} & \bmax{82}{10} & \gval{8}{4} & \gval{75}{7} & \gval{67}{8} & \gval{80}{12} & \gval{61}{10} \\
                 & \texttt{task4}   & \gval{6}{3} & \gval{15}{5} & \gval{35}{6} & \gval{2}{2} & \gval{24}{6} & \gval{29}{6} & \bmax{54}{8} & \gval{9}{3} \\
                 & \texttt{task5}   & \gval{44}{9} & \gval{37}{6} & \gval{71}{4} & \gval{7}{3} & \gval{78}{6} & \bmax{82}{6} & \bmax{82}{6} & \gval{61}{11} \\
                \multirow{-6}{*}{\texttt{cube-double}} & \texttt{agg. (5 tasks)}   & \gval{44}{4} & \gval{55}{2} & \gval{72}{4} & \gval{9}{2} & \gval{71}{2} & \gval{71}{3} & \bmax{81}{3} & \gval{58}{3} \\
                \midrule
                
                 & \texttt{task1}   & \gval{26}{20} & \gval{2}{3} & \gval{70}{18} & \gval{51}{22} & \gval{56}{22} & \gval{26}{11} & \gval{87}{8} & \ourmax{98}{4} \\
                 & \texttt{task2}   & \gval{1}{1} & \gval{0}{0} & \gval{16}{7} & \gval{16}{3} & \gval{15}{11} & \gval{14}{9} & \gval{42}{9} & \ourmax{75}{10} \\
                 & \texttt{task3}   & \gval{8}{6} & \gval{0}{0} & \gval{40}{8} & \gval{28}{8} & \gval{21}{11} & \gval{9}{6} & \gval{45}{9} & \ourmax{81}{11} \\
                 & \texttt{task4}   & \gval{2}{3} & \gval{0}{0} & \gval{15}{5} & \gval{1}{1} & \gval{4}{3} & \gval{3}{3} & \gval{30}{13} & \ourmax{68}{10} \\
                 & \texttt{task5}   & \gval{1}{1} & \gval{0}{0} & \gval{31}{12} & \gval{23}{8} & \gval{1}{1} & \gval{2}{4} & \bmax{48}{17} & \gval{37}{10} \\
                \multirow{-6}{*}{\texttt{cube-triple-10M}} & \texttt{agg. (5 tasks)}   & \gval{7}{5} & \gval{0}{1} & \gval{34}{6} & \gval{24}{7} & \gval{19}{6} & \gval{11}{4} & \gval{50}{5} & \ourmax{72}{4} \\
                \midrule
                
                 & \texttt{task1}   & \gval{36}{25} & \gval{7}{5} & \gval{36}{12} & \gval{15}{11} & \gval{70}{10} & \gval{45}{19} & \gval{66}{16} & \ourmax{90}{6} \\
                 & \texttt{task2}   & \gval{4}{5} & \gval{0}{1} & \gval{5}{3} & \gval{5}{7} & \gval{2}{4} & \gval{0}{0} & \gval{18}{11} & \ourmax{87}{6} \\
                 & \texttt{task3}   & \gval{3}{3} & \gval{0}{0} & \gval{2}{2} & \gval{6}{8} & \gval{16}{11} & \gval{2}{4} & \gval{11}{7} & \ourmax{62}{8} \\
                 & \texttt{task4}   & \gval{0}{0} & \gval{0}{0} & \gval{0}{1} & \gval{0}{0} & \gval{2}{2} & \gval{0}{0} & \gval{3}{2} & \ourmax{31}{17} \\
                 & \texttt{task5}   & \gval{0}{0} & \gval{0}{0} & \gval{1}{1} & \bmax{2}{3} & \gval{0}{0} & \gval{0}{0} & \gval{0}{0} & \gval{0}{0} \\
                \multirow{-6}{*}{\texttt{cube-quadruple-100M}} & \texttt{agg. (5 tasks)}   & \gval{9}{5} & \gval{1}{1} & \gval{9}{3} & \gval{6}{3} & \gval{18}{3} & \gval{9}{3} & \gval{19}{5} & \ourmax{54}{5} \\
                \midrule
                
                \texttt{all} & \texttt{agg. (50 tasks)} & \gavg{28} & \gavg{23} & \gavg{46} & \gavg{35} & \gavg{35} & \gavg{45} & \gavg{68} & \ouravg{75} \\
                
                \bottomrule
                \end{tabular}
            }
        \end{minipage}
    }
\end{table}

\begin{figure}[p]
\begin{center}
\includegraphics[width=\linewidth]{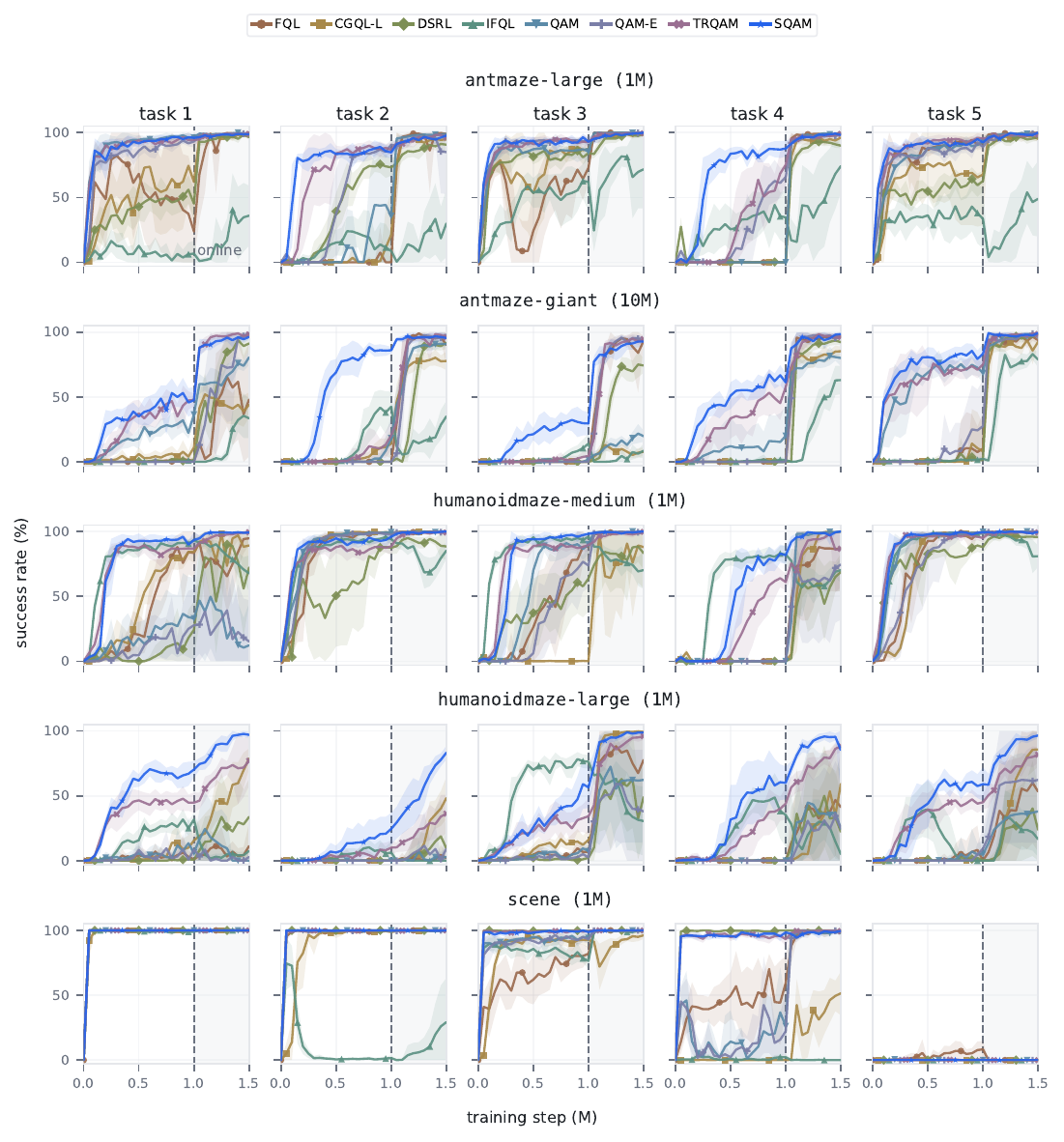}
\end{center}
\caption{Success rate curves on the locomotion domains and \texttt{scene}, with rows for domains and
columns for their five tasks. Offline training runs to 1M steps and online training to 1.5M, with the
dashed line at the transition. Eight seeds per curve, and bands are one standard deviation.}
\label{fig:bench-1}
\end{figure}

\begin{figure}[p]
\begin{center}
\includegraphics[width=\linewidth]{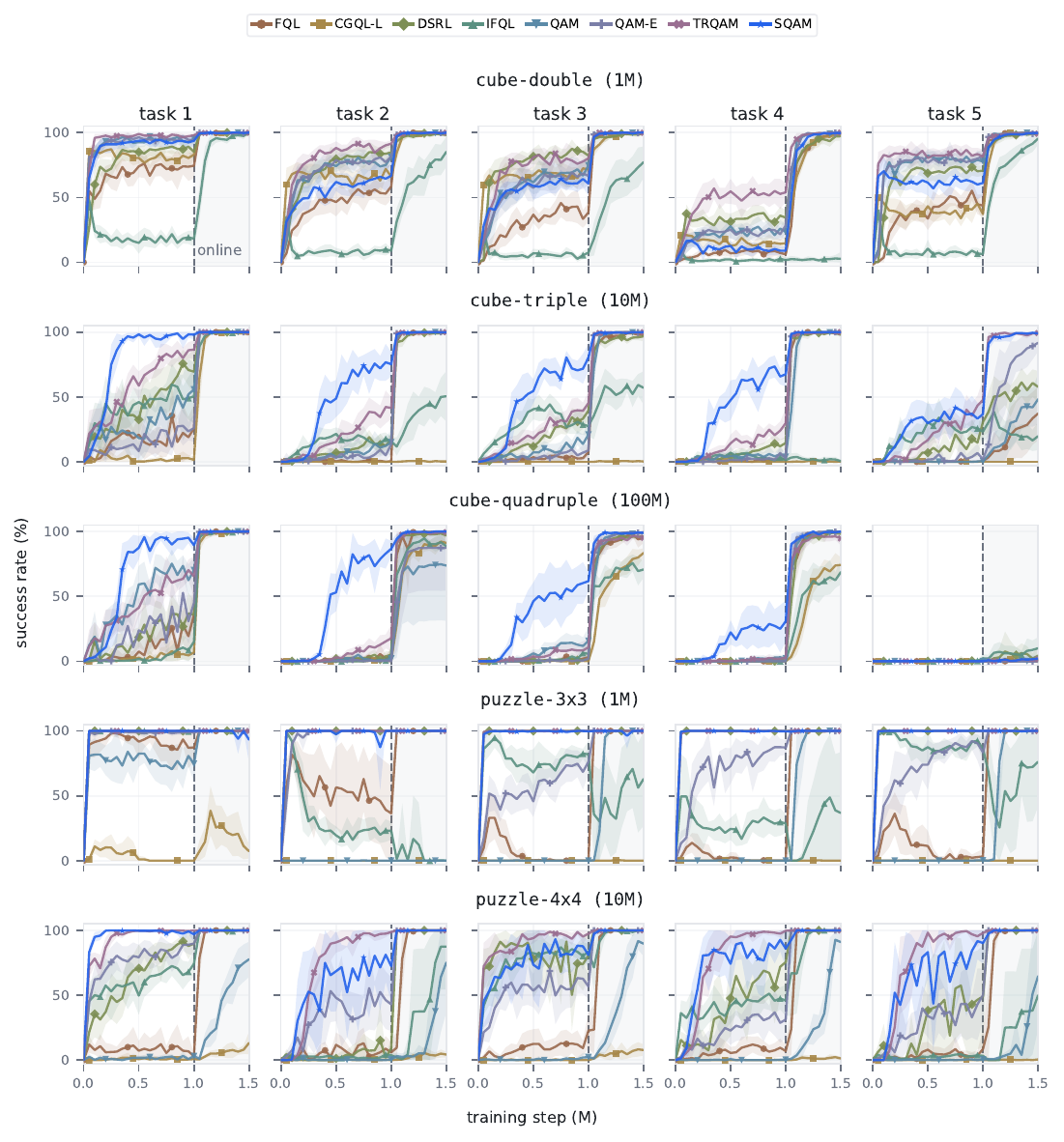}
\end{center}
\caption{Success rate curves on the \texttt{cube} and \texttt{puzzle} domains, with rows for domains and
columns for their five tasks. Offline training runs to 1M steps and online training to 1.5M, with the
dashed line at the transition. Eight seeds per curve, and bands are one standard deviation.}
\label{fig:bench-2}
\end{figure}

\newpage
\section{Limitations and future work}
\label{app:limitations}

We see three directions that this work leaves open.

\paragraph{Value learning for flow policies.} Our results suggest that the value learning scheme largely
determines how well the scalar adjoint performs, and \method{} addresses this with a value penalty added
to a standard critic. A value learning method designed for flow policies from the ground up may make
this dependence more reliable than a penalty on top of an existing critic.

\paragraph{Multi-task fine-tuning.} Updating the generative policy itself, rather than a separately
learned component, should matter most when a single policy serves many tasks. Our real-robot experiments
fine-tune the policy on each task separately. Fine-tuning a single policy across many tasks, as in
\citet{wang2026learningdeployingfleetscalereinforcement}, would test this benefit more directly.

\paragraph{Why the Jacobian concentrates on its diagonal.} We observe that the batch-averaged velocity
Jacobian of pretrained flow policies concentrates on its diagonal on every domain we measured, but we
do not explain why. A theoretical analysis of when this structure arises would clarify the conditions
under which the scalar adjoint applies.

\end{document}